\documentclass{article}

\usepackage{microtype}
\usepackage{graphicx}
\usepackage{subcaption}
\usepackage{booktabs} %
\usepackage{enumitem}
\usepackage{bbm}
\usepackage{makecell}

\usepackage[table]{xcolor}
\usepackage{colortbl}

\usepackage{calc}
\usepackage{subcaption}
\usepackage{hyperref}

\usepackage[accepted]{icml2024}

\usepackage{amsmath}
\usepackage{amssymb}
\usepackage{mathtools}
\usepackage{amsthm}
\usepackage{multirow}

\usepackage[capitalize,noabbrev]{cleveref}

\theoremstyle{plain}
\newtheorem{theorem}{Theorem}[section]

\newtheorem{lemma}[theorem]{Lemma}
\newtheorem{corollary}[theorem]{Corollary}
\theoremstyle{definition}
\newtheorem{definition}[theorem]{Definition}

\theoremstyle{remark}

\usepackage{pifont}%
\newcommand{\cmark}{\ding{51}}%
\newcommand{\xmark}{\ding{55}}%

\newcommand{\edit}[1]{{#1}}

\newcommand{\cubsub}{CUB-sub}

\newcommand{\cub}{CUB}

\newcommand{\clevr}{CLEVR}

\newcommand{\hamfull}{HAM10000}
\newcommand{\ham}{Ham}
\newcommand{\nlp}{Truth}
\newcommand{\nlpsub}{Truth-sub}
\newcommand{\nlptwo}{News}

\newcommand{\ourmethod}{CCE}
\newcommand{\ourmethodfull}{Compositional Concept Extraction}

\newcommand{\cn}[1]{\ensuremath{\small{\text{\texttt{#1}}}}}

\usepackage{amsmath,amsfonts,bm}

\newcommand{\remove}[1]{}
\newcommand{\update}[1]{#1}

\def\eqref#1{equation~\ref{#1}}

\def\1{\bm{1}}

\DeclareMathAlphabet{\mathsfit}{\encodingdefault}{\sfdefault}{m}{sl}
\SetMathAlphabet{\mathsfit}{bold}{\encodingdefault}{\sfdefault}{bx}{n}

\def\sC{{\mathbb{C}}}

\def\sR{{\mathbb{R}}}

\usepackage[textsize=tiny]{todonotes}

\icmltitlerunning{Towards Compositionality in Concept Learning}

\begin{document}

\twocolumn[

\icmltitle{Towards Compositionality in Concept Learning}

\icmlsetsymbol{equal}{*}

\begin{icmlauthorlist}
\icmlauthor{Adam Stein}{p}
\icmlauthor{Aaditya Naik}{p}
\icmlauthor{Yinjun Wu}{pku}
\icmlauthor{Mayur Naik}{p}
\icmlauthor{Eric Wong}{p}
\end{icmlauthorlist}

\icmlaffiliation{p}{Department of Computer and Information Science, University of Pennsylvania, Pennsylvania, USA}
\icmlaffiliation{pku}{School of Computer Science, Peking University, Beijing, China}

\icmlcorrespondingauthor{Adam Stein}{steinad@seas.upenn.edu}

\icmlkeywords{Machine Learning, ICML}

\vskip 0.3in
]

\printAffiliationsAndNotice{} %

\begin{abstract}
Concept-based interpretability methods offer a lens into the internals of foundation models by decomposing their embeddings into high-level concepts. These concept representations are most useful when they are \textit{compositional}, meaning that the individual concepts compose to explain the full sample. We show that existing unsupervised concept extraction methods find concepts which are not compositional. To automatically discover compositional concept representations, we identify two salient properties of such representations, and propose \ourmethodfull{} (\ourmethod{}) for finding concepts which obey these properties.
We evaluate \ourmethod{} on five different datasets over image and text data.
Our evaluation shows that \ourmethod{} finds more compositional concept representations than baselines and yields better accuracy on four downstream classification tasks.
\footnote{Code and data are available at \url{https://github.com/adaminsky/compositional_concepts}.}

\end{abstract}

\vspace{-0.2in}
\section{Introduction}
\label{sec:intro}

Foundation models continue to enable impressive performance gains across a variety of domains, tasks, and data modalities \citep{srivastava2023beyond}.
However, their black-box nature severely limits the ability to debug, monitor, control, and trust them \citep{turpin2024language, tamkin2023evaluating, schaeffer2024emergent}.

\begin{figure}[t!]
    \centering
        \includegraphics[width=\columnwidth]{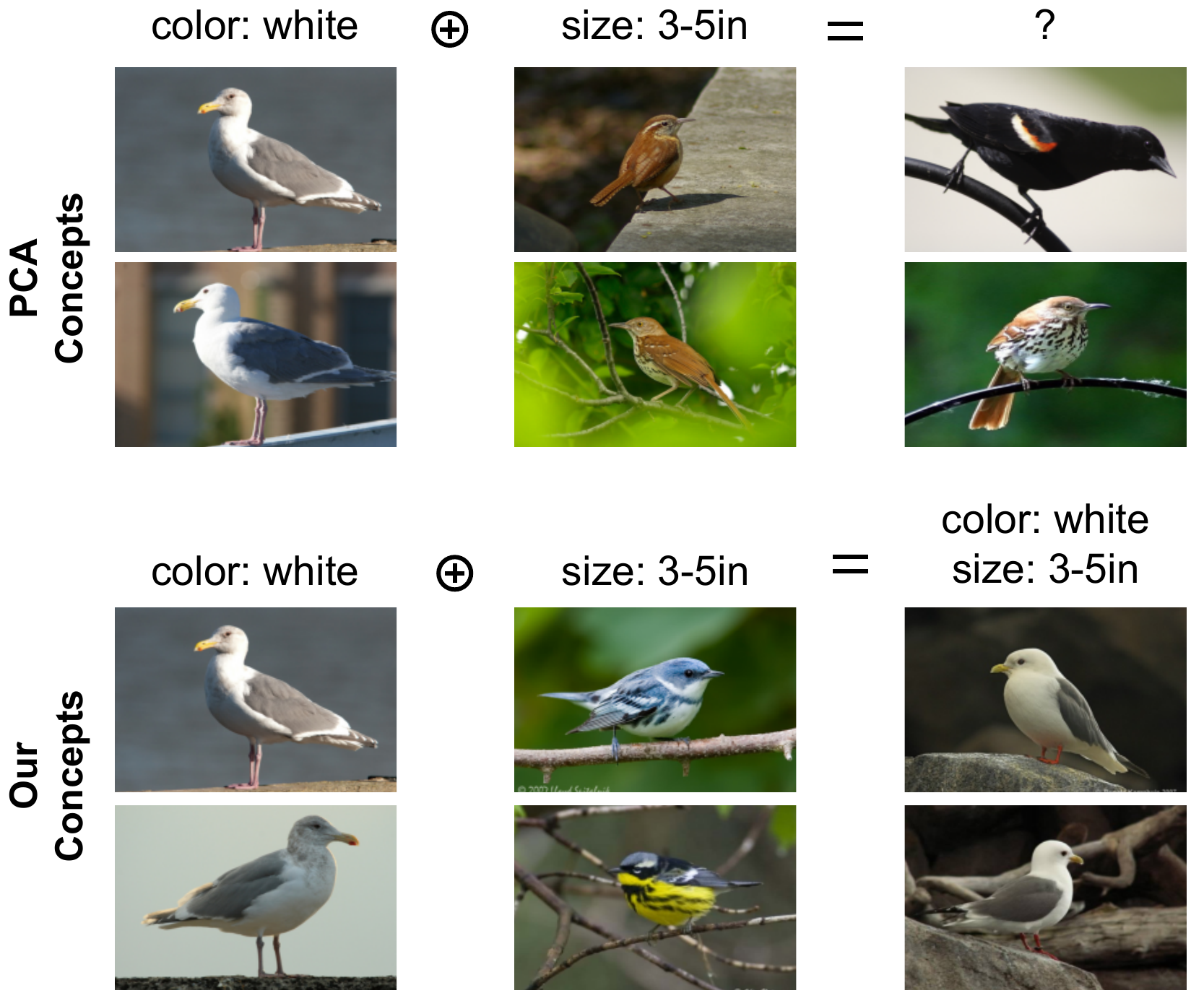}
    \caption{
    We illustrate the issue of concept compositionality with respect to concepts extracted from the embeddings of the CLIP model over the CUB dataset.
    Specifically, we visualize the concepts \cn{white birds} and \cn{small birds} learned by PCA~\cite{repe} and \ourmethod{} along with their compositions.
    We show the top two images that best represent each concept.
    Ideally, composing the \cn{white birds} and \cn{small birds} concepts should result in a concept representing small white birds.
    This is not the case with the concepts learned by PCA.
    On the other hand, the concepts extracted by \ourmethod{} are composable, as shown by the images of small white birds that best represent the resulting concept.
    }
    \label{fig:fig1}
    \vspace{-0.2in}
\end{figure}

Concept-based explanations \citep{tcav, zhou2018interpretable} are a promising approach that seeks to explain a model's behavior using individual concepts such as object attributes (e.g.~\cn{striped}) or linguistic sentiment (e.g.~\cn{happiness}).
Decomposing a model's learned representation can derive these concepts.
For instance, a model's embedding of a dog image may decompose into the \linebreak  sum of concept vectors representing its fur, snout, and tail.

Existing works based on methods such as PCA~\cite{repe} or KMeans~\cite{ace} extract such concept vectors reasonably well for basic concepts.
For instance, Figure~\ref{fig:fig1} shows images from the CUB~\cite{wah2011caltech} dataset containing concepts extracted by PCA from the CLIP~\cite{radford2021learning} model.
These techniques are able to correctly extract the representations of concepts like \cn{white birds} and \cn{small birds},
however, composing them by adding their representations does not yield the representation of the concept of \cn{small white birds}.

The \textit{compositionality} of concepts is vital for several use cases. First, model predictions can be explained by combining concepts \citep{abid2022meaningfully}.
Compositional concepts also allow for editing fine-grained model behavior, like improving the truthfulness of an LLM without compromising other behaviors \citep{repe}.
Models can also be trained to compose basic concepts for new tasks, e.g. using concepts for beak shapes, wing colors, and environments to classify bird species~\citep{pcbm}.

In this paper, we study the unsupervised extraction of compositional concepts.
Existing work does not directly evaluate the compositionality of extracted concepts, but rather focuses on the individual concept representations.
We therefore evaluate the compositionality of concepts extracted by existing unsupervised approaches.

For this purpose, we first validate the compositionality of ground-truth representations of concepts in controlled settings.
We observe that concepts can be grouped into \textit{attributes}, where each attribute consists of concepts over some common property, such as the color of objects or the shape of objects.
Concepts from different attributes (e.g. \cn{blue} and \cn{cube}) can be composed, while those from the same attribute (e.g. \cn{red} and \cn{green}) cannot.
We also observe that the concepts from different attributes are roughly orthogonal, while those from the same attribute are not.
We prove in a generalized setting that these properties are crucial for the compositionality of concepts.
Since existing approaches do not enforce these properties, they often extract non-composable concept representations.

To extract compositional concepts in an unsupervised manner, we propose \ourmethodfull{} (\ourmethod{}).
Our key insight is to search for entire subspaces of concepts at once instead of individual concepts, allowing
\ourmethod{} to enforce the aforementioned properties of compositional concepts.
We show that \ourmethod{} recovers the representation of known compositional concepts better than existing approaches, can discover compositional concepts in existing image and text datasets, and the discovered concepts improve downstream classification accuracy.

We thus summarize the contributions of our paper:
\begin{itemize}[noitemsep,nolistsep, leftmargin=*, topsep=0pt]
    \item We study concept-based explanations of foundation models from the lens of compositionality---a property desirable for many use-cases. We observe that concept representations extracted by state-of-the-art methods fail to compose, and set out to remedy this problem.
    \item We validate that models can in fact represent concepts compositionally in embedding space. We identify two salient properties of compositional concept representations that existing methods fail to satisfy.
    \item We prove in a generalized setting that the identified properties are necessary for compositionality.  We present a novel method called Compositional Concept Extraction (CCE) that guarantees to yield concept representations that satisfy these properties by construction.
    \item We demonstrate that CCE extracts more compositional concepts than baselines on vision and language datasets, and they improve downstream performance.

\end{itemize}

\section{Concepts and Compositionality}

\textit{\textbf{Concept Representations.}}
In machine learning, concepts are symbols that are assigned some human-interpretable meaning, often used to explain predictions made by models.

A concept extractor $E$ extracts concepts from the intermediate representation of some pretrained model $M$ over a dataset $D$.
$E(M, D)$ thus yields a set of \textit{concept vectors} representing the concepts $C = \{ c_1, \ldots, c_i \}$.
Concept vectors are denoted as $R(c)$, where $R: \sC \rightarrow \sR^d$ is the concept representation function, $\sC$ is the set of all possible concepts, and $\sR^d$ is an embedding space in some dimension $d$.
The set of extracted concepts $C$ can be grouped into mutually exclusive \textit{attributes} $A_1, \ldots A_k$ each containing concepts about some common property such that $C = \bigcup_{i=1}^k A_i$.

To measure the presence (or degree of expression) of a concept in a sample's embedding, we borrow the following definition of concept score from \citep{NEURIPS2020_ecb287ff}.

\begin{definition}
    (Concept Score) For a concept $c\in\sC$ and concept representation function $R:\sC\rightarrow \sR^d$, a sample embedding $z\in\sR^d$ has \textit{concept score} $s(z, c) = S_{\cos} (z, R(c))$ where $S_{\cos}$ is the \textit{cosine similarity} function.
\end{definition}

Existing work makes use of concept scores to quantify the presence of concepts on a per-sample basis.
This has uses in several applications, such as creating concept bottleneck models where a sample's embedding is converted to concept scores used for classification \citep{pcbm}, and sorting samples by a concept \citep{tcav}.

\textit{\textbf{Compositionality.}}
Following work on compositional representations \citep{andreas2018measuring} and pretrained embeddings \citep{trager2023linear}, we define the compositionality of concept representations.

\begin{definition}
\label{def:composition}
    (Compositional Concept Representations)
    For concepts $c_i, c_j \in \sC$,
    the concept representation $R:~\sC~\rightarrow~\sR^d$ is compositional if for some $w_{c_i}, w_{c_j} \in \sR^+$,
    \vspace{-0.05in}
    $$
    R(c_i \cup c_j) = w_{c_i}R(c_i) + w_{c_j}R(c_j).
    $$

    \vspace{-0.11in}
    
\end{definition}

In other words, the representation of the composition of concepts corresponds to the weighted sum of the individual concept vectors in the embedding space.

Furthermore, concept scores for the concepts satisfying Definition~\ref{def:composition} also behave compositionally, since each concept score quantifies the presence of that concept in a sample.

\begin{lemma}
\label{lemma:comp-score}
    For compositional concepts $c_i, c_j \in \sC$,
    the concept score of their composition $c_k = c_i \cup c_j$ over a sample embedding $z \in \sR^d$ is the composition of the concept scores of $c_i$ and $c_j$, weighted by $w_{c_i},w_{c_j}\in\sR^+$:
    \vspace{-0.05in}
    \begin{align*}
        s(z, c_k) &= w_{c_i}s(z, c_i) + w_{c_j}s(z, c_j).
    \end{align*}
    \vspace{-0.31in}
\end{lemma}

Since concept scores are used for several downstream tasks discussed above, this property about the compositionality of concept scores can simplify such tasks and improve the overall performance on them.

Besides finding compositional concepts, we also want to explain embeddings based on the concepts which compose it.
Prior work also performs a decomposition into a sum of concept representations \citep{zhou2018interpretable}, but we modify the definition of such a decomposition so that a sample embedding is composed of only the concept representations that are truly present for the sample.

\begin{definition}
\label{def:decomp}
    (Concept-based Decomposition)
    Consider a sample that is associated with a set of concepts $C \subseteq \sC$, such that each attribute $A_i \subseteq C$ contains exactly one concept.
    A concept representation $R: \sC\rightarrow \sR^d$ decomposes that sample's embedding $z_i \in \sR^d$ if it can be expressed as the weighted sum of the sample's associated concepts:

\begin{small}
\vspace{-0.2in}
\begin{align*}
    z_i = \sum_{c \in C} \lambda_{i,c} R(c) \text{, such that } \lambda_{i,j} > 0.
\end{align*}
\vspace{-0.26in}
\end{small}
\end{definition}
As an example, consider the CLEVR dataset~\cite{clevr} consisting of images of objects of different shapes and colors.
A concept extractor for a vision model may extract the set of concepts
$C_{\text{CLEVR}} = \{ \cn{\{red\}}, \cn{\{blue\}}, \cn{\{cube\}}, \cn{\{sphere\}}\}$.
$C_{\text{CLEVR}}$ can also be grouped into attributes $A_1 = \{ \cn{\{red\}}, \cn{\{blue\}} \}$ and $A_2 = \{ \cn{\{cube\}}, \cn{\{sphere\}}\}$ containing color and shape concepts respectively.
As such, a composite concept like \cn{\{red, sphere\}} can be represented as the weighted sum of $R(\cn{\{red\}})$ and $R(\cn{\{sphere\}})$.

\begin{figure*}
\centering
\hfill
\begin{subfigure}[b]{0.5\textwidth}
    \small
    \centering
    \begin{tabular}{lrrrr}
    \toprule
       \rowcolor{gray!20} Method &  CLEVR & CUB-sub & \nlpsub\\
        \midrule

GT & \cellcolor{pink!100} 1.000 $\pm$ 0.000 & \cellcolor{pink!100} 0.808 $\pm$ 0.000 & \cellcolor{pink!100} 0.625 $\pm$ 0.000 \\
PCA & \cellcolor{pink!85} 0.981 $\pm$ 0.000 & \cellcolor{pink!100} 0.663 $\pm$ 0.000 & \cellcolor{pink!42} 0.467 $\pm$ 0.000 \\
ACE & \cellcolor{pink!57} 0.834 $\pm$ 0.029 & \cellcolor{pink!85} 0.651 $\pm$ 0.011 & \cellcolor{pink!100} 0.551 $\pm$ 0.017 \\
DictLearn & \cellcolor{pink!71} 0.891 $\pm$ 0.005 & \cellcolor{pink!71} 0.650 $\pm$ 0.010 & \cellcolor{pink!71} 0.533 $\pm$ 0.006 \\
SemiNMF & \cellcolor{pink!42} 0.780 $\pm$ 0.029 & \cellcolor{pink!42} 0.629 $\pm$ 0.029 & \cellcolor{pink!57} 0.525 $\pm$ 0.050 \\
CT & \cellcolor{pink!28} 0.575 $\pm$ 0.039 & \cellcolor{pink!28} 0.510 $\pm$ 0.003 & \cellcolor{pink!14} 0.428 $\pm$ 0.055 \\
Random & \cellcolor{pink!14} 0.568 $\pm$ 0.087 & \cellcolor{pink!14} 0.445 $\pm$ 0.079 & \cellcolor{pink!28} 0.461 $\pm$ 0.034 \\
\ourmethod{} & \cellcolor{pink!100} 1.000 $\pm$ 0.000 & \cellcolor{pink!57} 0.648 $\pm$ 0.008 & \cellcolor{pink!85} 0.545 $\pm$ 0.004 \\

    \bottomrule
    \end{tabular}
    \caption{MAP score of predicting concept compositions.%
    }
    \label{tab:comp-ap}
\end{subfigure}
\hfill
  \begin{subfigure}[b]{0.4\textwidth}
  \centering
    \includegraphics[width=0.65\columnwidth]{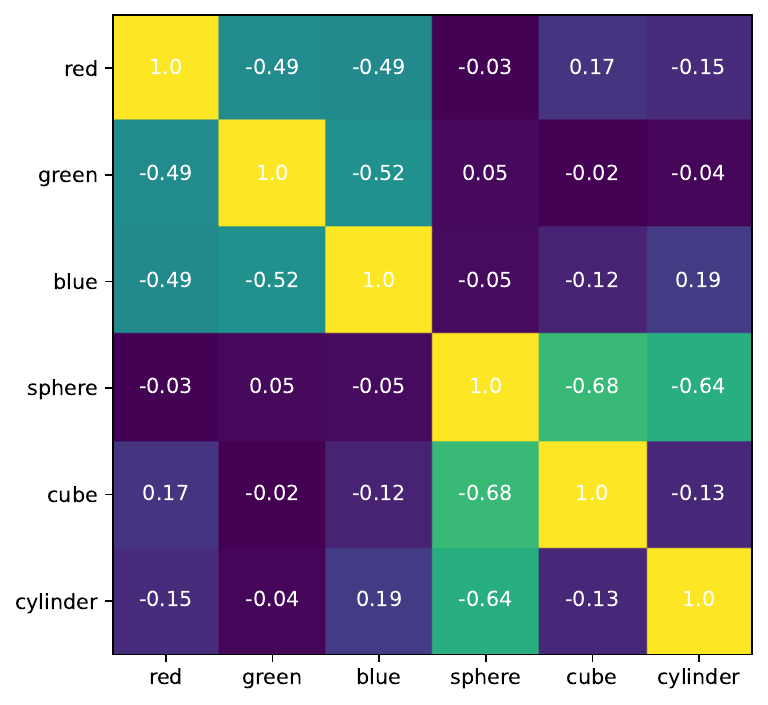}
    \caption{Cosine similarities between \clevr{} concepts.}
    \label{fig:clevr-gt}
  \end{subfigure}
\hfill
    \caption{Compositionality of ground-truth concepts compared with concepts extracted by existing approaches and \ourmethod{}. Figure~\ref{tab:comp-ap} shows that the ground-truth concepts (GT) are quite compositional, but existing methods are not.
    Figure~\ref{fig:clevr-gt} shows the cosine similarities between pairs of ground-truth concepts for the \clevr{} dataset. The darker blue cells represent concepts that are orthogonal, while the lighter yellow ones represent non-orthogonal ones. We observe that concepts tend to be more orthogonal if they belong to different attributes.}
    \label{fig:gt-orth-all}
    \vspace{-0.1in}
\end{figure*}

\section{Evaluating Concept Compositionality}\label{sec: concept_compose}

In this section, we validate the compositionality of ground-truth concept representations and evaluate the same for concepts extracted using existing approaches. 
We first discuss our controlled setting and show that concept representations from the CLIP model are compositional.
We then evaluate the compositionality of concepts extracted by existing approaches.
Finally, we outline the necessary properties of compositional concept representations.

\subsection{Setup}\label{sec: motivation_setup}

In order to validate the compositionality of ground-truth concepts, we focus on concepts extracted from subsets of the \clevr{} \citep{clevr}, \cub{} \citep{wah2011caltech}, and \nlp{} \citep{azaria2023internal} datasets, all of which have labelled attributes with compositional structure.

We follow a setup similar to \cite{lewis2022does} for the synthetic CLEVR~\cite{clevr} dataset and consider images with single objects labelled as one of three shapes (sphere, cube, or cylinder) and one of three colors (red, green, or blue).
We also consider a subset of the CUB dataset consisting of bird images labelled as one of three colors and one of three sizes.
Finally, we consider a subset of the \nlp{}~\cite{zou2023representation} dataset consisting of facts relating to one of three topics and labelled true or false.

\subsection{Ground-Truth Concept Compositionality}
\label{sec:gt-comp}

We evaluate the compositionality of ground-truth concept representations learned by the CLIP model over each labelled dataset.
Since these representations are not provided, for each concept, we consider the mean of the model's embeddings for samples belonging to that concept as a surrogate of its true representation \citep{repe}.

For example, for the CLEVR dataset, we extract the ground-truth representation of the \cn{red} concept by calculating the mean of all sample embeddings of images with red objects.
We similarly extract the ground-truth representations for the other two color concepts, the three shape concepts, and composite concepts like \cn{\{red, sphere\}}, for a total of 15 concepts.
We repeat this process for each dataset.

As stated in Lemma~\ref{lemma:comp-score}, the concept score for a composite of two concepts is the weighted sum of the concept scores of each concept.
This implies that a linear model should be able to predict the concept score for a composed concept given the concept scores for each of the base concepts.
We thus train a linear model to predict the presence or absence of a composed concept given its base concepts.
We measure the average precision of the model for each composed concept, and report the mean average precision (MAP) score in Table~\ref{tab:comp-ap} for each dataset.
We see that in all cases, the ground truth (GT) concepts have high MAP (up to 0.971 for \clevr{}) when predicting concept compositions from their components, meaning the ground-truth concept representations are reasonably compositional.

\subsection{Compositionality Issues with Existing Methods}
\label{sec:comp-issues}

We next study the compositionality of concept representations discovered by existing unsupervised concept extraction methods.
We train a linear model similar to the one described in Section~\ref{sec:gt-comp}, but with concepts extracted by baseline methods instead of the ground truths.
From the MAP results in Table~\ref{tab:comp-ap} we see that all the baselines have significantly lower compositionality than the ground-truth.

This is the case even for techniques that extract the concepts reasonably well, i.e. where the extracted concepts are able to discriminate between positive and negative samples of that concept.
For each dataset and concept extraction method, we calculate the ROC-AUC score to measure the ability of the extracted concept to perform such a discrimination.
We provide the full ROC-AUC results in Appendix~\ref{app:roc-auc}.
In the case of NMF, despite this score averaging as high as 0.907 for the \clevr{} dataset, the extracted concepts are not compositional.
This implies that finding concept representations simply based on their ability to discriminate positive and negative samples of a concept does not mean that those representations will compose as expected.

We further demonstrate this point with a toy illustration in Figure~\ref{fig:incorrect-comp}. This figure depicts four perfectly composed concepts at the top, and four incorrectly composed concepts at the bottom, even though each concept is perfectly discriminative of the samples with the concept. 
Therefore, we must ensure that we explicitly extract compositional concepts.

\begin{figure}[t]
    \centering
    \includegraphics[width=0.65\columnwidth]{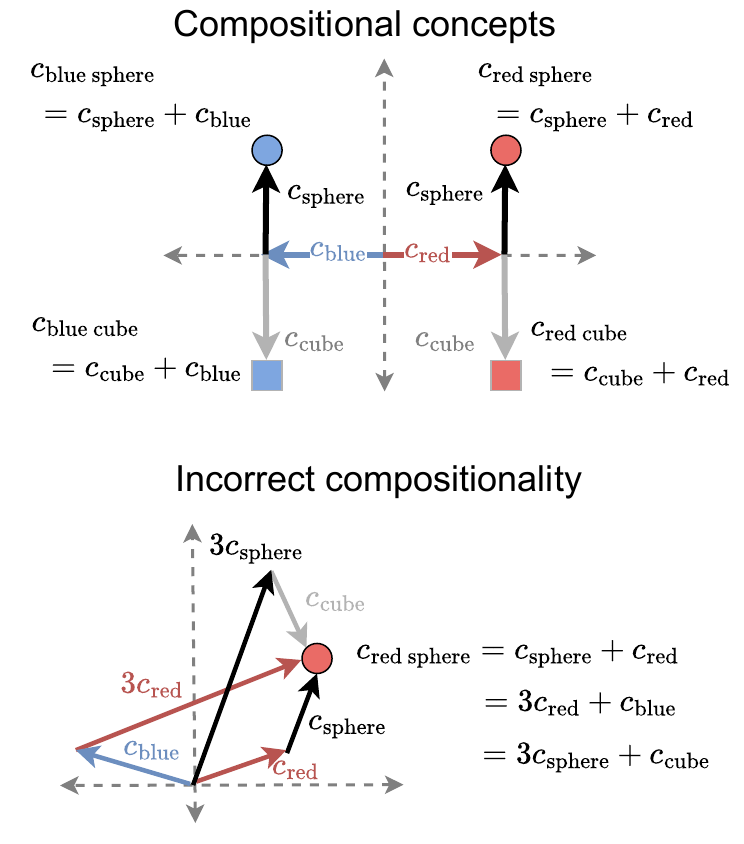}
    \vspace{-0.1in}
    \caption{
    Illustration of concepts on a dataset of cubes and spheres that are either red or blue. The concepts on the top are compositional while those on the bottom are not. Even though the concepts on the bottom can perfectly represent the four samples, they still fail to compose properly. For instance, the composition of the \cn{red} and \cn{blue} concepts can form the \cn{\{red, sphere\}} concept even though the \cn{blue} concept is not present in a red sphere.}
    \label{fig:incorrect-comp}
    \vspace{-0.15in}
\end{figure}

\subsection{Desired Properties of Compositional Concepts}

To extract compositional concepts, we must first identify characteristics of such concepts.
Since the ground-truth concepts were compositional, we investigate the salient characteristics of those concepts.

Consider the ground-truth concepts for the \clevr{} dataset.
In order to understand the relationship between different ground-truth concepts and their compositionality, we center the sample embeddings and visualize cosine similarities between pairs of these concepts in Figure~\ref{fig:clevr-gt}.
We observe that the ground-truth representations of color concepts are roughly orthogonal (cosine similarity near 0) to those of shape concepts.
In contrast, the representations of concepts within the same attribute, such as the \cn{red} and \cn{blue} concepts, are non-orthogonal.
Furthermore, the orthogonal concepts are also those that can compose to form new concepts, since they lie in different attributes.
For instance, the \cn{red} and \cn{sphere} concepts are orthogonal, and can compose to form the \cn{\{red, sphere\}} concept, while the \cn{red} concept can't compose with the \cn{blue} concept.

We visualize the same for the \cubsub{} and \nlpsub{} datasets in  Appendix~\ref{sec:comp-gt}, and empirically observe the following trend over all three datasets:
concept representations from different attributes are roughly orthogonal while those from the same attribute are non-orthogonal.
Also, the orthogonal concepts tend to be compositional, while the non-orthogonal ones can't be composed.

Orthogonality is a generally helpful property for several use cases, such as disentangling concepts in embedding space \citep{chen2020concept}.
Some approaches therefore try to enforce orthogonality on the concepts being extracted.
Table~\ref{tab:properties} summarizes existing unsupervised approaches for concept extraction
and whether the method enforces any orthogonality constraints (Ortho.) between concepts of different attributes and allows for non-orthogonality between those of the same attribute (Corr.).
We see that these approaches allow for only one of the two, but not both.

\begin{table}
\small
    \centering
    \caption{Properties enforced by unsupervised concept extraction.}
    \label{tab:properties}
    \begin{tabular}{p{1.25cm}p{3.5cm}ll}
    \toprule
        Method & Example & Ortho. & Corr.\\
        \midrule
        PCA & RepE \citep{repe} & \cmark & \xmark\\
        KMeans & ACE \citep{ace} & \xmark & \cmark\\
        Dictionary-Learning & TransformerVis \citep{yun2021transformer} & \xmark & \cmark\\
        NMF & CRAFT \citep{craft} & \xmark & \cmark\\
        Custom & Concept Tf \citep{rigotti2021attention} & \xmark & \cmark\\
        \midrule
        Custom & \ourmethod{} (Ours) & \cmark & \cmark\\
    \bottomrule
    \end{tabular}
\end{table}

We now formally prove that the observed properties regarding concept compositionality hold in a generalized setting.

\begin{theorem}
\label{theorem}
    For some dataset, \edit{consider two attributes $A$ and $A'$ where $A$ has $l$ concepts $c_1, \dots c_{l}$ and $A'$ has $l'$ concepts $c'_1, \dots c'_{l'}$.} %
    Assuming that for each compositional concept \edit{$c=\{c_{i},c'_{j}\}$}, its representation \edit{$v_{i,j}$}, follows a spherical normal distribution with zero mean and unit covariance, i.e. \edit{$v_{i,j}$} $\sim N(\mathbf{0}, \mathbf{I}^d)$,
    the following statements are true with high probability for a large dimension $d$:
    \begin{itemize}[noitemsep,nolistsep, leftmargin=*]
        \item \edit{There exists $c_1, c_2 \in A$ and $c'_1, c'_2 \in A'$ such that the representations of these base concepts are non orthogonal.}
        \item \edit{For all $c_1\in A$ and $c_2\in A'$, the representations of $c_1$ and $c_2$ are} orthogonal with high probability.
    \end{itemize}
\end{theorem}

We show the proof in Appendix~\ref{app:thm}. The takeaway from this result is that compositional concepts will be roughly orthogonal, while concepts of the same attribute may not be orthogonal. In addition, we show in Corollary~\ref{thm:reverse-corollary} that given concepts which follow the consequent of the above theorem, that the concepts will have compositional concept representations, meaning the representations of composite concepts consist of a sum of their component base concept representations, as defined in Definition~\ref{def:composition}. We leverage this to design an unsupervised concept extraction method which can find compositional concepts when they exist.

\section{Compositional Concept Extraction (\ourmethod)}

\begin{figure*}
    \centering
    \includegraphics[width=0.7\textwidth]{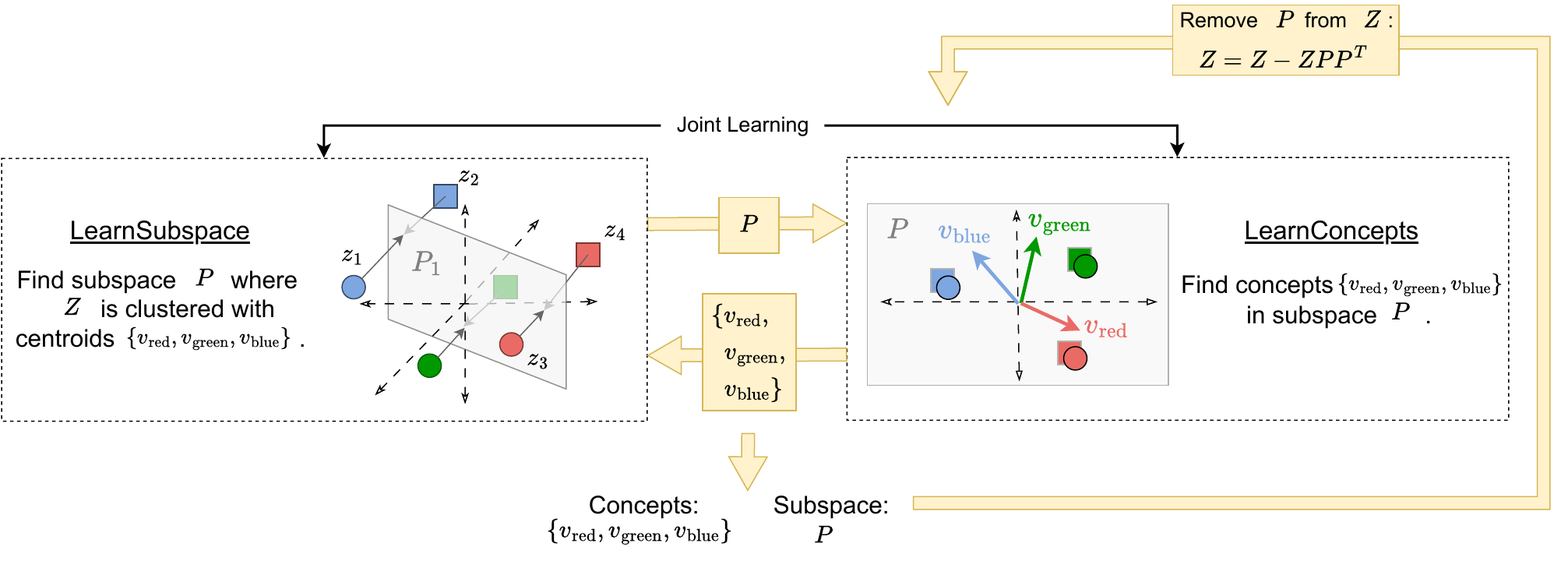}
    \vspace{-0.1in}
    \caption{Finding color concepts in one iteration of \ourmethod{}, which can be proceeded by finding other concepts, such as shapes.}
    \label{fig:method}
    \vspace{-0.1in}
\end{figure*}

\begin{small}
\begin{algorithm}[ht]
\small
   \caption{\ourmethodfull{}}
   \label{alg:method}
\begin{algorithmic}
   \STATE {\bfseries Input:} embeddings $Z$, num. attr. $M$, concepts per attr. $K$, subspace dimension $S$
   \STATE Initialize concepts $C=\{\}$
   \FOR{$m=1\dots M$}
        \STATE Initialize $P\in\sR^{d\times S}$ such that $P^TP=I$.
        \STATE Initialize $K$ concepts $V = \{v_1,\dots, v_K\}$.
        \REPEAT
            \STATE $P = \text{LearnSubspace}(P, Z, V)$
            \STATE $V = \text{LearnConcepts}(ZP, K)$
        \UNTIL{Converged}
   \STATE $C = C \cup V$
   \STATE $Z = Z - ZPP^T$
   \ENDFOR
   \STATE Return $C$
\end{algorithmic}
\end{algorithm}    
\end{small}

To achieve this orthogonality property between concepts, we propose \ourmethod{}, summarized in Algorithm \ref{alg:method} and visualized in Figure~\ref{fig:method}. As the outer loop of the algorithm suggests, once we find concepts for an attribute in a subspace $P$, we remove that subspace using orthogonal rejection and find concepts in a new subspace. 
This enforces orthogonality between the discovered subspaces, thus respecting the orthogonality property described in Section 3. 

To discover concepts within each attribute, we employ a two-step process consisting of LearnSubspace and LearnConcepts, as illustrated in Figure~\ref{fig:method}. The LearnSubspace step, shown on the left, is given a clustering of the data (in terms of centroids $V$) and optimizes a subspace, defined by $P\in\sR^{d\times S}$, so that the data in this subspace ($ZP$) becomes well clustered according to the fixed centroids $V$. In the next step, LearnConcepts, shown on the right, we identify concepts by performing spherical K-Means clustering on $ZP$, the data within subspace $P$.

This clustering process is performed within a learned subspace and the subspace is learned according to the learned clustering. Therefore, we jointly learn the subspace $P$ and the clustering centroids $V$.
Specifically, for LearnSubspace,
we employ the Silhouette score \citep{rousseeuw1987silhouettes} to quantify how well clustered the projected data $ZP$ is given a cluster assignment $L$ determined by the centroids from spherical K-Means clustering. The Silhouette score measures the ratio of average within cluster distance to average between cluster distance. Since the Silhouette score is differentiable, once we fix a clustering $L$ from LearnConcepts, we perform a step of gradient ascent in LearnSubspace to increase the Silhouette score. Thus, we solve the following optimization problem by iteratively fixing $P$ to learn $L$ (with LearnConcepts) and then fixing $L$ to learn $P$ by a gradient step (with LearnSubspace) until convergence:

\begin{small}
\vspace{-0.28in}
\begin{align*}
    \arg\max\nolimits_{P, L} \text{Sil}(ZP, L).
\end{align*}    
\vspace{-0.3in}
\end{small}

We further observe that simply maximizing the above objective leads to overfitting issues since projecting the learned cluster centroids from LearnConcepts back to the original space may not necessarily correspond to cluster centroids in the original space. 
Therefore, in the LearnSubspace step we additionally try to match the cluster centroids learned within the subspace and projected out to the original space to the centroids of the clusters in the original space. This is integrated into the above full objective function as a regularization term, i.e.:

\begin{small}
\vspace{-0.25in}
\begin{align*}
    \arg\max\nolimits_{P, L} \left(\text{Sil}(ZP, L) + \sum\nolimits_k S_{\cos{}}(C_k P^T, \hat{C}_k)\right),
\end{align*}    
\end{small}
where $C_k$ represents the clustering centroids in the subspace $P$ while $\hat{C}_k = \frac{1}{\sum\nolimits_i \mathbbm{1}[L_i = k]}\sum\nolimits_i \mathbbm{1}[L_i = k] Z_i$ represents the clustering centroids in the original space.

\section{Experiments}
\label{sec:exp}

\subsection{Experimental Setup}\label{sec: exp_setup}
\textbf{Datasets and Models.}
We evaluate using five datasets across vision and language settings: CLEVR~\cite{clevr} (vision), \cub\ \citep{wah2011caltech} (vision),  \hamfull\  \citep{tschandl2018ham10000} (vision), \nlp\ \citep{zou2023representation} (language), and \nlptwo\ \citep{misc_twenty_newsgroups_113} (language). We perform experiments on both {\em controlled} and {\em full} settings.
In the controlled setting, we follow the same configuration as Section \ref{sec: motivation_setup} for the \clevr{}, \cub\ and \nlp\ datasets. Further information on our datasets is included in Appendix~\ref{app:datasets}.
The {\em full setting} considers all samples from the \cub, \ham, \nlp, and \nlptwo\ datasets.

For the image datasets, we obtain sample representations from the CLIP model \citep{radford2021learning} while for the NLP dataset, this is achieved with Llama-2 13B Chat model \citep{touvron2023llama}. We also perform ablation studies on the choices of different models in Appendix~\ref{app:ablation}.

\textbf{Baseline Methods.}
Since the concept representations are learned by \ourmethod{} in an unsupervised manner, we therefore primarily compare \ourmethod{} against the following state-of-the-art unsupervised concept extraction methods, i.e., PCA \citep{repe}, NMF \citep{craft}, ACE (KMeans) \citep{ace}, and Dictionary Learning \citep{bricken2023monosemanticity, yun2021transformer}. In addition, we include a Random baseline where we randomly initialize concept vectors from a normal distribution of mean zero and variance one.

Recent studies like Concept Transformer \cite{rigotti2021attention} explore how to jointly learn concept representations and perform training of downstream classification tasks with learned concept representations. Hence, we treat Concept Transformer (Concept Tf) \cite{rigotti2021attention} as
another baseline. Note that Concept Tf can optionally incorporate concept labels as additional supervisions, which are not considered in our experiments for fair comparison.

\begin{figure*}[ht!]
    \centering
    \begin{subfigure}[b]{0.49\textwidth}
        \centering
        \includegraphics[width=0.8\textwidth]{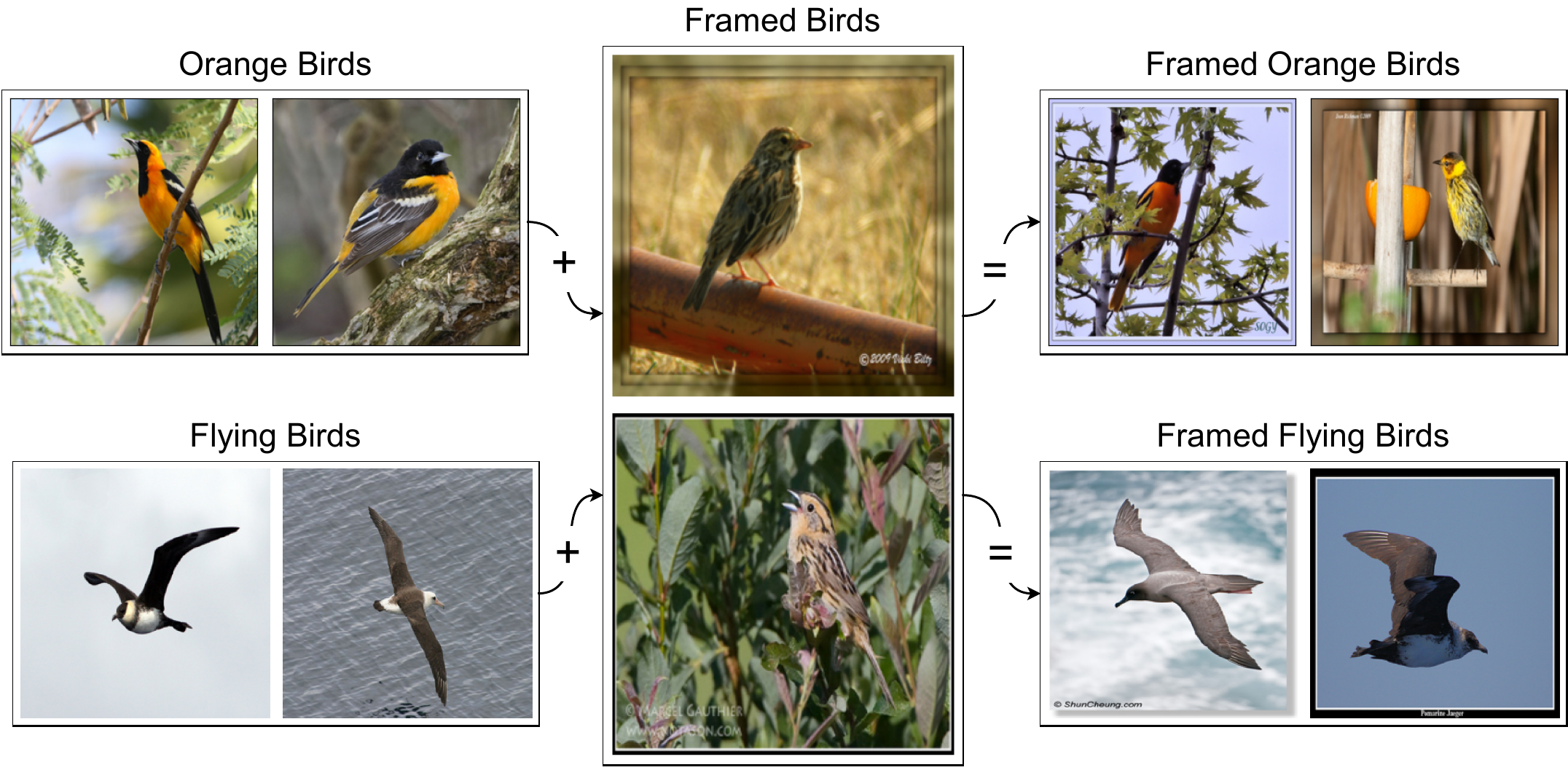}
        \caption{}
        \label{fig:qual_cub_1}
    \end{subfigure}%
    \hfill
    \begin{subfigure}[b]{0.49\textwidth}
        \centering
        \includegraphics[width=0.8\textwidth]{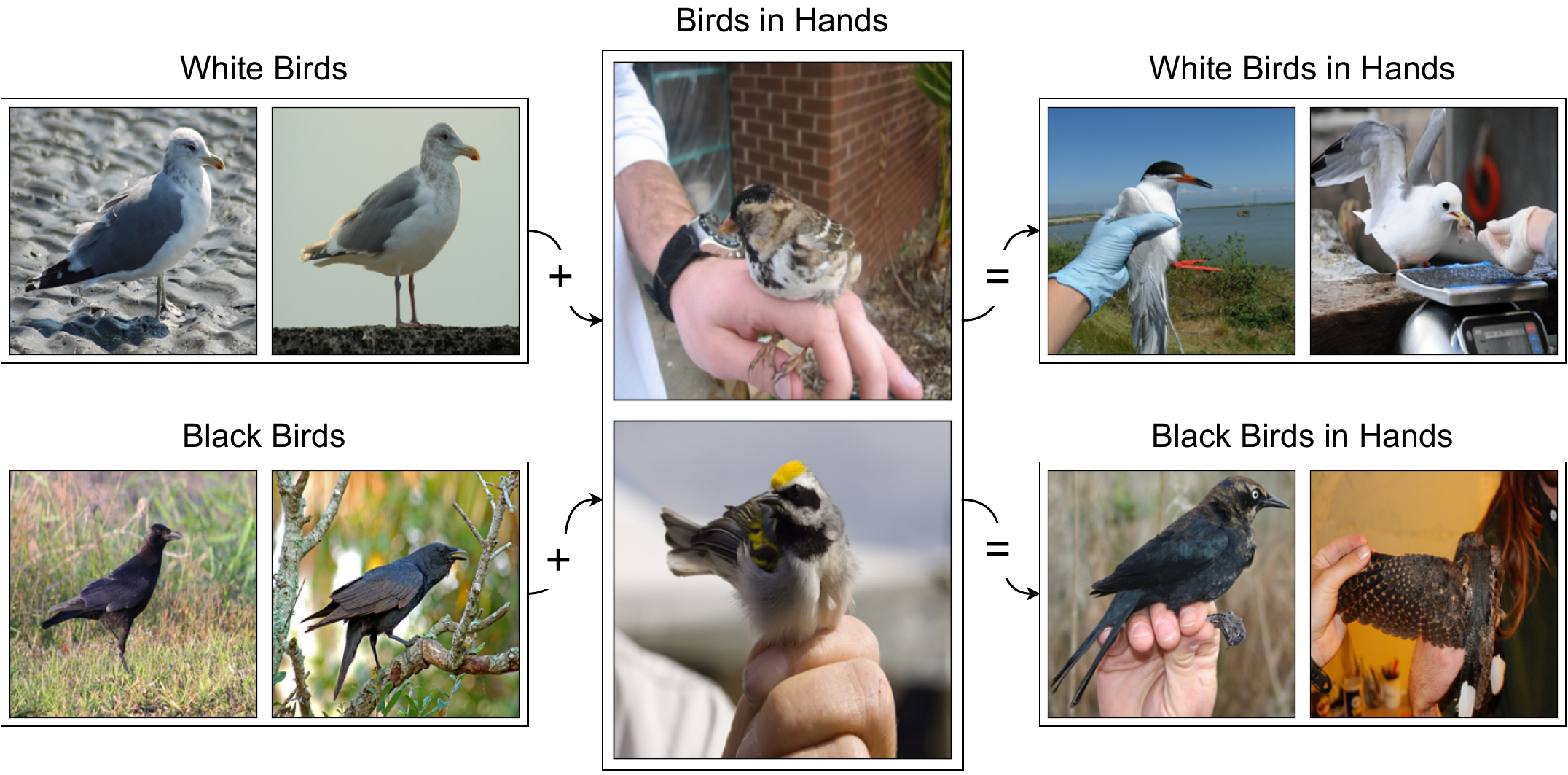}
        \caption{}
        \label{fig:qual_cub_2}
    \end{subfigure}%
    
    \begin{subfigure}[b]{0.49\textwidth}
        \centering
        \includegraphics[width=\textwidth]{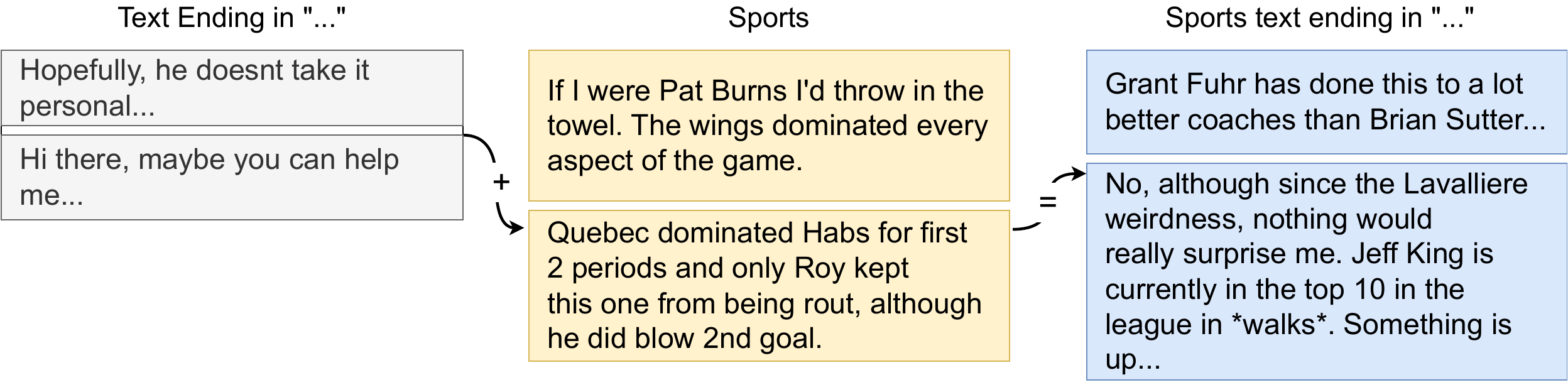}
        \caption{}
        \label{fig:qual_news_1}
    \end{subfigure}%
    \hfill
    \begin{subfigure}[b]{0.49\textwidth}
        \centering
        \includegraphics[width=\textwidth]{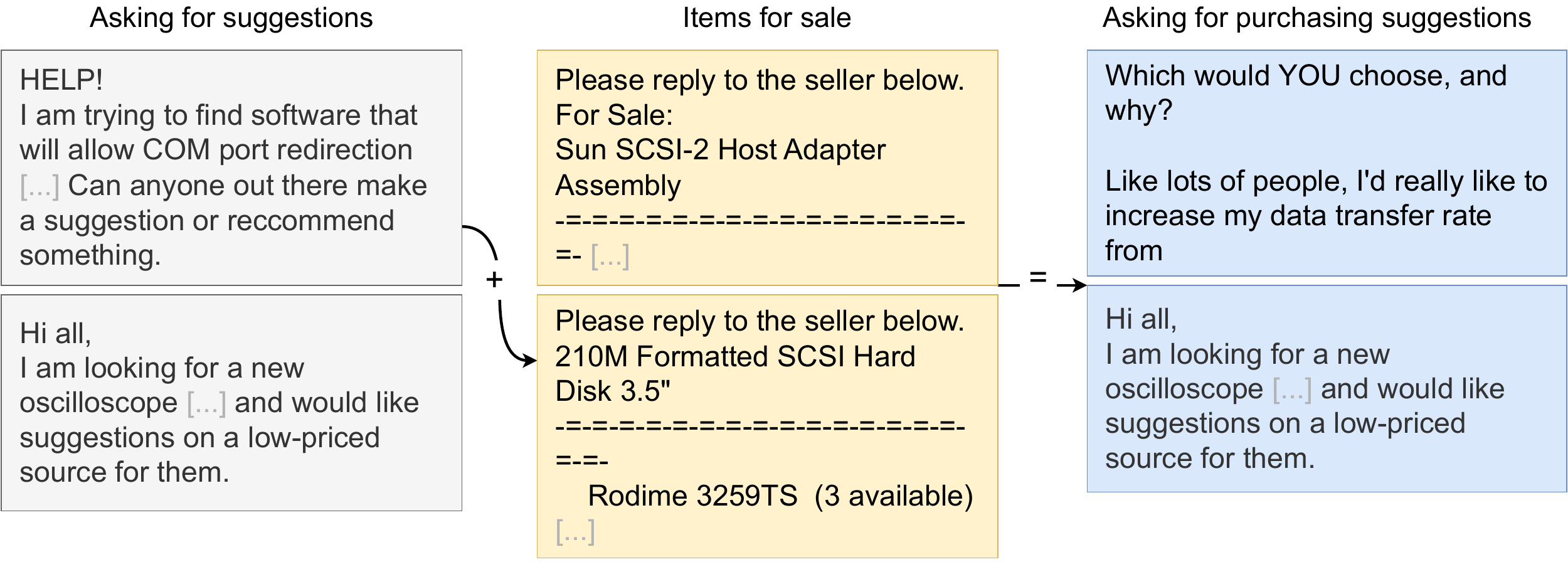}
        \caption{}
        \label{fig:qual_news_2}
    \end{subfigure}%
    \vspace{-0.15in}
    \caption{Examples of compositional concepts identified by \ourmethod{}. Figures \ref{fig:qual_cub_1} and \ref{fig:qual_cub_2} are from the \cub\ dataset while Figures \ref{fig:qual_news_1} and \ref{fig:qual_news_2} are from the \nlptwo\ dataset. These figures suggest that \ourmethod{} can not only discover new meaningful concepts outside the ground-truth concepts, such as the \cn{Birds in Hands} concept in Figure \ref{fig:qual_cub_2}, but also compose these concepts correctly, e.g. \cn{White Birds} + \cn{Birds in Hands} =  \cn{White Birds in Hands}.}\label{fig: concept_qualitative}
\end{figure*}

\textbf{Experiment Design.} We aim to answer these questions regarding the quality of the learned concept representations:
\begin{enumerate}[leftmargin=*,nolistsep, nosep, label=\textbf{RQ\arabic*}]
  \item In the controlled setting with known compositional ground-truth concept representations, does \ourmethod{} compose concepts more effectively than baselines? \label{rq1}
  \item In the full setting where the ground-truth concepts are typically unknown, can \ourmethod{} successfully discover new and meaningful compositional concepts? \label{rq2}
  \item In both controlled and full settings, how can the learned compositional concept representations impact downstream performance? \label{rq3}
\end{enumerate}

To address \ref{rq1}, we evaluate the compositionality score \citep{andreas2018measuring} on the concept representations extracted by \ourmethod{} and the baselines, which is defined as follows:

\begin{definition}\label{comp-score}
    (Compositionality Score) Given a dataset $D$ consisting of embeddings $z\in\sR^d$, their associated ground-truth concepts $C\subset\sC$, and a concept representation function $R: \sC \rightarrow \sR^d$ obtained from a concept extractor,
    the compositionality score is the following:

\begin{small}
\vspace{-0.3in}
    \begin{align*}
        \min_{\Lambda \ge 0} \frac{1}{|D|}\sum_{(z, C)\in D} \left\|z - \sum_{i=1}^{|C|} \Lambda_{z, i}R(C_i)\right\|
    \end{align*}
\vspace{-0.2in}
\end{small}
\end{definition}
Intuitively speaking, for a sample embedding $z$, this metric quantifies how much $z$ can be reconstructed by composing a list of concept representation $R(c_i)$'s that correspond to the $i_{th}$ ground-truth concepts of $z$. Each $R(c_i)$ is weighted by a coefficient $\Lambda_{z, i}$, which is determined by optimizing the above formula with respect to all $\Lambda_{z, i}$.

In addition, for each ground-truth concept, we also report the cosine similarity between the learned concept representation $R(c_i)$ and the corresponding ground-truth representation.

To study \ref{rq2} for the full setting, we primarily perform qualitative studies to identify whether \ourmethod{} is capable of discovering reasonable compositional concepts. Specifically, for each learned concept representation, we assign a name to the concept by inspecting the ten images with the top concept score. Then for each pair of the learned concepts, we first identify those samples with the highest concept scores. Then, we sum the two concept representations, and find the samples with largest concept score for this aggregated representation. By investigating these examples, we visually examine whether the composition is reasonable or not.

Lastly, we answer \ref{rq3} by evaluating the downstream classification performance with the learned concept representations.  
Specifically, we follow \citet{pcbm} to learn a linear classifier by predicting class labels with the concept scores of a sample. We further report the performance of training a linear classifier on sample embeddings without involving any concepts, denoted by ``No concept''.

\subsection{Experimental Results}

\begin{table} %
\vspace{-0.1in}
    \centering
    \caption{Compositionality Scores (lower is better).}
    \footnotesize
    \label{tab:performance-table}
    \begin{tabular}{lrrr}
    \toprule
        \rowcolor{gray!20} & \clevr{} &  \cubsub\   & \nlpsub{} \\
        \midrule

GT & \cellcolor{pink!100} 3.162 $\pm$ 0.000 & \cellcolor{pink!85} 0.462 $\pm$ 0.000 & \cellcolor{pink!57} 3.743 $\pm$ 0.000 \\
PCA & \cellcolor{pink!42} 3.684 $\pm$ 0.000 & \cellcolor{pink!71} 0.472 $\pm$ 0.000 & \cellcolor{pink!28} 3.975 $\pm$ 0.000 \\
ACE & \cellcolor{pink!57} 3.474 $\pm$ 0.134 & \cellcolor{pink!42} 0.496 $\pm$ 0.007 & \cellcolor{pink!71} 3.727 $\pm$ 0.032 \\
DictLearn & \cellcolor{pink!71} 3.367 $\pm$ 0.016 & \cellcolor{pink!28} 0.498 $\pm$ 0.002 & \cellcolor{pink!85} 3.708 $\pm$ 0.007 \\
SemiNMF & \cellcolor{pink!28} 3.716 $\pm$ 0.053 & \cellcolor{pink!57} 0.495 $\pm$ 0.004 & \cellcolor{pink!42} 3.781 $\pm$ 0.074 \\
CT & \cellcolor{pink!0} 4.929 $\pm$ 0.002 & \cellcolor{pink!14} 0.545 $\pm$ 0.000 & \cellcolor{pink!0} 4.348 $\pm$ 0.000 \\
Random & \cellcolor{pink!14} 4.925 $\pm$ 0.000 & \cellcolor{pink!0} 0.545 $\pm$ 0.000 & \cellcolor{pink!14} 4.348 $\pm$ 0.000 \\
\ourmethod{} & \cellcolor{pink!85} 3.163 $\pm$ 0.000 & \cellcolor{pink!100} 0.459 $\pm$ 0.004 & \cellcolor{pink!100} 3.689 $\pm$ 0.002 \\

    \bottomrule
    \end{tabular}
    \vspace{-0.1in}
\end{table}

\begin{figure*}[h!]
  \centering
      \begin{subfigure}[b]{0.24\textwidth}
    \includegraphics[width=\textwidth]{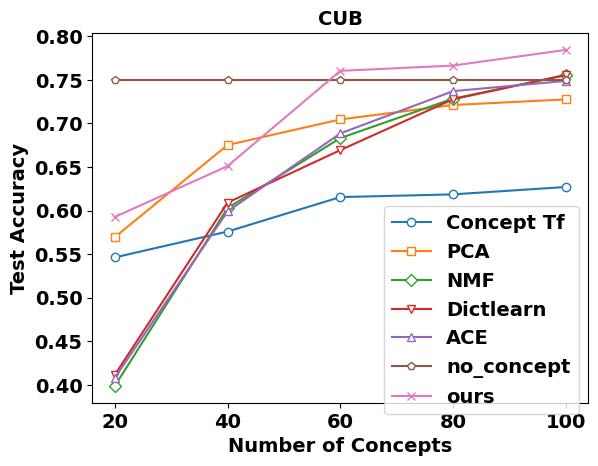}
    \label{fig:downstream_res1}
  \end{subfigure}
  \hfill
  \begin{subfigure}[b]{0.24\textwidth}
    \includegraphics[width=\textwidth]{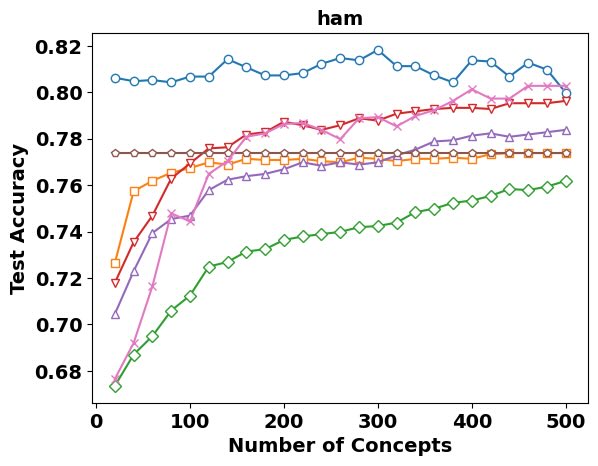}
    \label{fig:downstream_res2}
  \end{subfigure}
    \hfill
  \begin{subfigure}[b]{0.24\textwidth}
    \includegraphics[width=\textwidth]{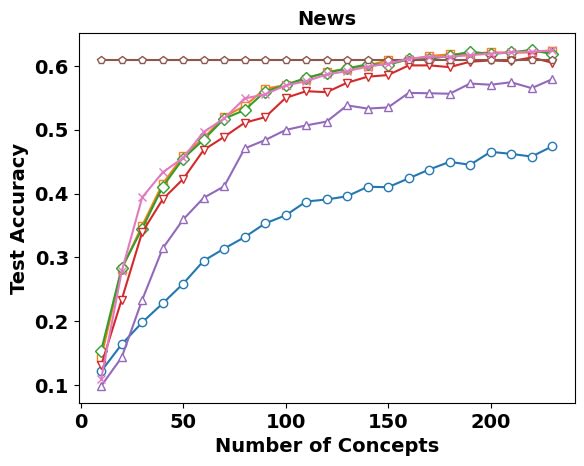}
    \label{fig:downstream_res3}
  \end{subfigure}
    \hfill
  \begin{subfigure}[b]{0.24\textwidth}
    \includegraphics[width=\textwidth]{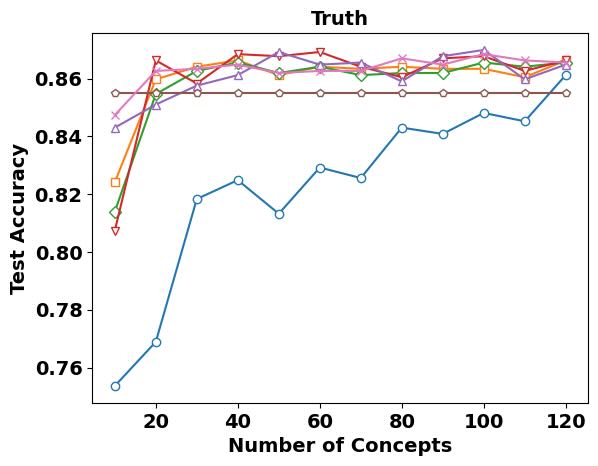}
    \label{fig:downstream_res3}
  \end{subfigure}
    \vspace{-0.25in}
  \caption{Downstream classification accuracy on the full setting.}
  \label{fig:downstream_res}

  \vspace{-0.1in}
\end{figure*}

\textbf{Compositionality in Controlled Settings.}
We first evaluate the compositionality scores on the \clevr{}, \cubsub{}, and \nlpsub{} datasets and report them in Table \ref{tab:performance-table}. In all cases, \ourmethod{} obtains the best score compared to the baselines, indicating the advantage of \ourmethod{} in discovering compositional concepts. Moreover, \ourmethod{}'s scores are comparable to those of the ground-truth concept representations.
This shows that the concepts learned by \ourmethod{} almost align with the ground-truth concept representations.

This is further supported by the results in Table \ref{tab:cosine_mean}. This table summarizes the cosine similarities between the ground-truth concept representations and the ones learned by the baselines and \ourmethod{}.
Again, the concepts learned by \ourmethod{} are the closest to the ground truths.
Note that some baselines like Dictlearn also produce highly accurate concept representations.
However, as Table~\ref{tab:performance-table} shows, their compositions fail to be consistent with the ground truths.

\textbf{Compositionality in Real Data Settings.}
To address \ref{rq2}, we perform some qualitative studies on compositional concepts discovered by \ourmethod{} on the \cub\ and \nlptwo\ dataset, which are visualized in Figure \ref{fig: concept_qualitative}. As shown in this figure, \ourmethod{} is capable of identifying reasonable concepts, such as \cn{White Birds}, \cn{Framed Birds} and  \cn{Text Ending in ``...''}. Some
 of these concepts are even beyond the ground-truth concept labels that are provided by the dataset itself. For example, \ourmethod{} identifies the ``Birds in Hands'' concept which is not labeled in the \cub\ dataset. But its top activated samples are images with a bird in someone's hand (see Figure \ref{fig:qual_cub_2}). Furthermore, the composition of those learned concepts is also representative of the properties of each concept. For example, in Figure \ref{fig:qual_news_1}, the composition of the concept \cn{Text Ending in ``...''} and \cn{Sports} represents sentences about ``sports" ending in ``...''.

\begin{table}[h!]
    \vspace{-0.1in}
    \centering
    \caption{The average cosine similarity between individual learned concept representations and the ground truth
    (higher is better).}
    \label{tab:cosine_mean}
    \footnotesize
    \begin{tabular}{lrrr}
    \toprule
        \rowcolor{gray!20} &\clevr{} & \cubsub\   & \nlpsub{} \\
        \midrule

PCA & \cellcolor{pink!42} 0.580 $\pm$ 0.000 & \cellcolor{pink!42} 0.503 $\pm$ 0.000 & \cellcolor{pink!42} 0.459 $\pm$ 0.000 \\
ACE & \cellcolor{pink!57} 0.728 $\pm$ 0.009 & \cellcolor{pink!85} 0.719 $\pm$ 0.016 & \cellcolor{pink!57} 0.648 $\pm$ 0.007 \\
DictLearn & \cellcolor{pink!85} 0.745 $\pm$ 0.003 & \cellcolor{pink!57} 0.661 $\pm$ 0.010 & \cellcolor{pink!85} 0.686 $\pm$ 0.007 \\
SemiNMF & \cellcolor{pink!71} 0.732 $\pm$ 0.014 & \cellcolor{pink!71} 0.696 $\pm$ 0.002 & \cellcolor{pink!71} 0.673 $\pm$ 0.052 \\
CT & \cellcolor{pink!14} 0.044 $\pm$ 0.009 & \cellcolor{pink!28} 0.066 $\pm$ 0.001 & \cellcolor{pink!14} 0.019 $\pm$ 0.002 \\
Random & \cellcolor{pink!28} 0.059 $\pm$ 0.003 & \cellcolor{pink!14} 0.043 $\pm$ 0.011 & \cellcolor{pink!28} 0.024 $\pm$ 0.001 \\
\ourmethod{} & \cellcolor{pink!100} 0.992 $\pm$ 0.000 & \cellcolor{pink!100} 0.770 $\pm$ 0.001 & \cellcolor{pink!100} 0.804 $\pm$ 0.001 \\

        \bottomrule
    \end{tabular}
    \vspace{-0.15in}
\end{table}

\textbf{Downstream Performance Analysis.} For \ref{rq3}, we studied the impact of the extracted compositional concepts on downstream performance across all datasets in the full setting. Throughout the experiments, we observe that the total number of concepts is a crucial factor in determining the performance. Therefore, we also vary this number and report the performance numbers accordingly for all datasets and methods in Figure \ref{fig:downstream_res}. As this figure suggests, across all the datasets, despite the poor performance with a small number of concepts, \ourmethod{} gradually gains performance with an increasing number of concepts, eventually outperforming all the unsupervised baseline methods. 

Also, it is worth noting that \ourmethod{} outperforms Concept Tf most times and is on par with it in the worst case (see the experimental results on the ham dataset with 500 concepts). This thus indicates the performance advantage of \ourmethod{} even in the absence of supervision from downstream tasks.

Furthermore, \ourmethod{} discovers concept representations by performing a series of linear transformations on top of the sample embeddings. But by comparing against ``No concept'' where sample embeddings are directly used for downstream tasks, \ourmethod{} can even outperform it by a large margin on \cub\ and \ham\ dataset. This implies that the concept representations extracted by \ourmethod{} might be more relevant to the downstream classification tasks than the raw embeddings.

\section{Related Work}
\label{sec:relwork}

\textbf{Concept-based Interpretability.}
Concept-based interpretability encompasses the building of models using human-interpretable concepts \citep{koh2020concept, espinosa2022concept, pcbm} and extracting such concepts post-hoc from models \citep{tcav, zhou2018interpretable}. In either case, how do we choose which concepts to use?
Some existing work specifies concepts using human supervision to select and provide their labels \citep{tcav}, large-scale concept annotation datasets \citep{bau2017network}, general knowledge bases \citep{pcbm}, and large language models \citep{yang2023language}. Another line of work uses regularization \citep{wong2021leveraging}, or other inductive biases \citep{rigotti2021attention} to learn concepts during standard supervised training of a model. Finally, there is work which leverages unsupervised methods to automatically discover concepts \citep{ace, craft, yun2021transformer, bricken2023monosemanticity} which is the approach taken in this paper. Unlike existing unsupervised concept learning methods which focus on properties such as faithfulness \citep{ace} or human-meaningfulness \citep{craft}, we focus specifically on compositionality.

\textbf{Compositionality in Foundation Models.}
Since the observation of compositional word vectors by \citet{mikolov2013distributed} there has been interest in finding and utilizing compositional behavior of deep learning models. Existing work has leveraged insights from psychology and cognitive science to find concepts learned by generative models \citep{frankland2020concepts, lake2014towards}. Compositionality has been used to uncover and mitigate bias in word embeddings \citep{bolukbasi2016man}, edit classifier behavior \citep{santurkar2021editing}, and recently to monitor and control the behavior of foundational language \citep{todd2023function, repe} and vision models \citep{wang2023concept, kwon2022diffusion}. To the best of our knowledge, we are the first to evaluate compositionality of concept representations learned by unsupervised approaches and to propose a method to improve compositionality of discovered concepts.

\textbf{Compositional and Disentangled Representations.}
In representation learning, there is considerable effort to encourage \textit{disentangled} representations \citep{bengio2013representation, higgins2016beta, wang2022disentangled}. While disentanglement concerns how to distinguish separate concepts in embedding space, compositionality concerns what happens when separate concepts get combined. Existing work has shown that disentanglement and compositionality do not have to be correlated \citep{xu2022compositional}. Unlike representation learning, we start with a pretrained model and try to uncover the compositional concepts it learned.

\textbf{Structures beyond compositionality.}
This paper focuses on compositionality in concept-based interpretability, but other important structures include subpopulation, relational, and causal structures. Group, or subpopulation, structure has been used as a way to interpret datasets with existing work on automatically finding such structure \citep{blei2001latent} and explaining models with respect to this structure \citep{HavaldarSWU23}. In addition, existing work has developed methods to steer explanations to respect group structures \citep{stein2023rectifying}. Relational structures have also been studied as a lens into understanding the behavior of pretrained models \citep{todd2024function, lovering-pavlick-2022-unit, hill2019learning}. Beyond group and relational structures, recent work proposes a method to identify known causal structures in pretrained LLMs \citep{wu2024interpretability}.

\section{Limitations}
\label{sec:limitations}

We study the case where concepts compose compositionally, but concepts may also be non-compositional. For instance, the concepts of \cn{hot} and \cn{dog} do not compose to form the meaning of \cn{hot dog} \citep{zhai1997exploiting}. In addition, we supposed a flat concept structure, which does not distinguish between ``(small blue) car" and ``small (blue car)". We leave the study of such non-compositional and hierarchical concepts to future work.

Another limitation of unsupervised concept extraction is that discovered concept vectors are not associated with any name. We assign names to the concept through manual inspection of samples with a high concept score, but this can require significant effort with large numbers of concepts.

\section{Conclusion}
\label{sec:conclusion}
In this paper, we studied concept-based explanations of foundation models from the lens of compositionality. We validated that the ground-truth concepts extracted from these models are compositional while the existing unsupervised concept extraction methods usually fail to guarantee compositionality. To address this issue, we first identified two salient properties for compositional concept representations and designed a novel concept extraction method called \ourmethod\ that respects these properties by design. Through extensive experiments across vision and language datasets, we demonstrated that \ourmethod\ not only learns compositional concepts but also enhances downstream performance.

\section*{Acknowledgements}
This material is based upon work supported by the National Science Foundation Graduate Research Fellowship under Grand No. DGE-2236662, the Google Research Fellowship, and ``The Fundamental Research Funds for the Central Universities, Peking University''.

\section*{Impact Statement}
This paper presents work whose goal is to advance the field of Machine Learning. There are many potential societal consequences of our work, none which we feel must be specifically highlighted here.

\bibliography{refs}
\bibliographystyle{icml2024}

\newpage
\appendix
\onecolumn
\section{Proof of Lemma~\ref{lemma:comp-score}}
\label{lemma_proof}
\begin{proof}
Let $z\in\sR^d$ be a sample embedding, $R:\sC\rightarrow \sR^d$ be a compositional concept representation function, and $c_i, c_j\in\sC$ be two compositional concepts which compose as $c_k=c_i\cup c_j$. From Definition~\ref{comp-score}, the concept scores for $c_i$ and $c_j$ are the following:
\begin{align*}
    s(z, c_i) &= S_{\cos{}}(z, R(c_i))\\
    s(z, c_j) &= S_{\cos{}}(z, R(c_j)).
\end{align*}

The concept score for the composition $c_k$ can then be written as:
\begin{align*}
    s(z, c_k) &= s(z, c_i\cup c_j)\\
    &= S_{\cos{}}(z, R(c_i\cup c_j))\\
    &= S_{\cos{}}(z, w_{c_i}R(c_i) + w_{c_j}R(c_j)) & \text{(since $R$ is compositional)}\\
    &= \frac{z\cdot (w_{c_i}R(c_i) + w_{c_j}R(c_j))}{\|z\|\|w_{c_i}R(c_i) + w_{c_j}R(c_j)\|} & \text{(definition of cosine similarity)}\\
    &= \frac{z\cdot w_{c_i}R(c_i)}{\|z\|\|R(c_k)\|} + \frac{z\cdot w_{c_j}R(c_j)}{\|z\|\|R(c_k)\|}\\
    &= \frac{(w_{c_i}\|R(c_i)\|) z\cdot R(c_i)}{\|R(c_k)\|\|z\|\|R(c_i)\|} + \frac{(w_{c_j}\|R(c_j)\|) z\cdot R(c_j)}{\|R(c_k)\|\|z\|\|R(c_j)\|}\\
    &= \frac{w_{c_i}\|R(c_i)\|}{\|R(c_k)\|}S_{\cos{}}(z, R(c_i)) + \frac{w_{c_j}\|R(c_j)\|}{\|R(c_k)\|}S_{\cos{}}(z, R(c_j)) & \text{(definition of cosine similarity)}
\end{align*}
\end{proof}

\section{Proof of Theorem~\ref{theorem}}

\begin{lemma}[curse of dimensionality]\citep{wegner2021lecture}
    For a pair of vectors $\mathbf{x}$ and $\mathbf{y}$ randomly sampled from $N(0, \mathbf{I}^d)$, $\mathbf{x}$ and $\mathbf{y}$ are orthogonal with high probability for large enough $d$. Mathematically speaking, for a fixed small constant, $\epsilon$, the following inequality holds:
    \begin{align*}
        \mathbb{P}\left[|\langle \frac{\mathbf{x}}{|\mathbf{x}|}, \frac{\mathbf{y}}{|\mathbf{y}|}\rangle| \leq \epsilon \right] \geq 1 - \frac{M_1}{\sqrt{d}\epsilon} - \frac{M_2}{\sqrt{d}},
    \end{align*}
    where $M_1= 2$ and $M_2 =7$
\end{lemma}

\begin{lemma}[Gaussian Annulus Theorem]\citep{wegner2021lecture}\label{lemma:gaussian_bound}
    For a vector $v$ randomly sampled from $N(0, \mathbf{I}^d)$, $\|v\|$ is approaching $\sqrt{d}$ with high probability for large enough $d$. Mathematically speaking, the following inequality holds:
    \begin{align*}
        \mathbb{P}\left[|\|\mathbf{x}\| - \sqrt{d}| \leq \epsilon \right] \geq 2 \exp{(-M_3\epsilon^2)},
    \end{align*}
    in which $M_3 = \frac{1}{16}$
\end{lemma}

Based on the above two lemmas, for any two randomly sampled vectors $\mathbf{x}$ and $\mathbf{y}$ from $N(0, \mathbf{I}^d)$, the following equality holds with high probability:
\begin{align}\label{eq: inner_prod_bound}
    \langle \mathbf{x}, \mathbf{y} \rangle = o(d)
\end{align}

\update{
\begin{lemma}\label{lemma:base_concept_rep}
    As defined in Theorem~\ref{theorem}, for a composite concept $c=\{c_i,c'_j\}$, its representation is denoted by $v_{i,j}$, then the representation of the base concept $c_i$ belonging to attribute $A$ is:
    \begin{align*}
        v_i = \frac{1}{l'} \sum_{j=1}^{l'} v_{i,j}.
    \end{align*}
    Similarly, the representation of the base concept $c'_j\in A'$ is:
    \begin{align*}
        v'_j = \frac{1}{l} \sum_{i=1}^{l} v_{i,j}.
    \end{align*}
\end{lemma}
}

\begin{proof}
    \update{$v_i$ could be derived by calculating the mean of the representations of all samples with concept $c_i$ in the attribute $A$. Since those samples may have different concepts in the attribute $A'$, then the composite concept among these samples could be $\{c_i, c_1'\},\{c_i, c_2'\},\dots, \{c_i, c_l'\}$. Therefore, $v_i$ is derived by: 
    \begin{align*}
        v_i =  \frac{1}{N}\sum_{x\ \text{with concept $c_i$ in attribute A}}x = \frac{1}{N}\sum_{j=1}^{l'}\sum_{x\ \text{with concept $\{c_i,c_j'\}$}}x,
    \end{align*}    
    in which $N$ represents the number of samples with concept $c_i$ in attribute $A$. By further assuming that there is a large enough number of samples for each composite concept, this implies that the number of each composite concept is roughly the same, i.e., around $N/l'$. Then the above formula could be transformed to:
    \begin{align*}
        v_i = \frac{1}{N}\sum_{j=1}^{l'}\sum_{x\ \text{with concept $\{c_i,c_j'\}$}}x = \frac{1}{N}\sum_{j=1}^{l'}\frac{N}{l'} v_{i,j} = \frac{1}{l'}\sum_{j=1}^{l'}v_{i,j}.
    \end{align*}
    The last step in the above formula leverages the fact that $v_{i,j}$ is calculated by the mean of all samples belonging to composite concept $\{c_i,c_j'\}$.
    }

    \update{
    We can further illustrate this with one concrete example from the CLEVR dataset. By reusing the running example from Section \ref{sec: concept_compose}, we assume that there are three colors \{red, green, blue\} and three shapes \{sphere, cube, cylinder\} in the CLEVR dataset. By following the notations of Theorem~\ref{theorem}, the representation of a composite concept, say, $\{c_{\text{red}},c_{\text{sphere}}\}$, is represented by $v_{\text{red}, \text{sphere}}$. Then the representation of the base concept $\text{sphere}$ should be the mean of all samples belonging to this base concept. This can be derived by the mean of the samples belonging to the concept $\{c_{\text{red}},c_{\text{sphere}}\}$, the ones belonging to $\{c_{\text{green}},c_{\text{sphere}}\}$ and the ones belonging to $\{c_{\text{blue}},c_{\text{sphere}}\}$. Therefore, the representation of $c_{\text{sphere}}$ is denoted by:
    \begin{align*}
        v_{\text{sphere}} = \frac{1}{3}[v_{\text{red}, \text{sphere}} + v_{\text{green}, \text{sphere}} + v_{\text{blue}, \text{sphere}}]. 
    \end{align*}
    }

\end{proof}

We next present the formal proof of Theorem~\ref{theorem}:

\label{app:thm}
\begin{proof}
\update{We split our proof into two parts. The first part is for proving ``For the base concepts belonging to the same attribute, there exists at least one pair of non-orthogonal concepts.'' while the second part is for proving ``For any pair of base concepts from two different attributes, they are orthogonal with high probability.''}

\update{
\textbf{Part 1: There exists $c_1, c_2 \in A$ and $c'_1, c'_2 \in A'$ such that the representations of these base concepts are non orthogonal.}
}

\update{According to Lemma \ref{lemma:base_concept_rep}, the concept representation for the base concept $c_i$ (denoted by $\hat{v_i}$) is:
    \begin{align}\label{eq: base_concept_emb}
        \hat{v_i} = \frac{1}{l'} \sum_{j=1}^{l'} v_{i,j},
    \end{align}
which sums over all concepts in $A'$. 
}

\remove{First, we can derive the concept representation for each base concept $c_{i,t}$ (denoted by $\mu_{i,t}$) as follows:
    \begin{align}\label{eq: base_concept_emb}
        \mu_{i,t} =\frac{1}{l^{k-1}} \sum_{j_1}\sum_{j_2}\dots \sum_{j_{i-1}}\sum_{j_{i+1}}\dots \sum_{j_k} v_{j_1,j_2,j_3,\dots,j_{i-1},t,j_{i+1},\dots,j_k}.
    \end{align}   
}

Since we also want to perform centering operations over the entire dataset, then this suggests that we need to leverage the mean of all concepts, i.e.,:
\remove{\begin{align}\label{eq: all_mean}
        \mu =\frac{1}{l^{k}} \sum_{j_1}\sum_{j_2}\dots \sum_{j_k} v_{j_1,j_2,j_3,\dots,j_{i-1},j_i,j_{i+1},\dots,j_k}.
    \end{align}}
\update{
\begin{align}\label{eq: all_mean}
        \mu =\frac{1}{ll'} \sum_{i, j} v_{i, j}.
    \end{align}
}

Then after the centering operation, \remove{$\mu_{i,t}$}\update{$\hat{v_i}$} is transformed into:
\begin{align}\label{eq: centered_transform}
    v_i = \frac{\hat{v_i} - \mu}{\sigma}.
\end{align}
In the formula above, we use $\sigma$ to represent the standard deviation vector calculated over the entire dataset. 

\remove{
Then let us fix $i$ and sum up all $\mu_{i,t}'$ over all $t$, which yields:
\begin{align}\label{eq: centering_vec}
\begin{split}
    & \sum_{t=1}^l \mu_{i,t}'= \sum_{t=1}^l \frac{\mu_{i,t} - \mu}{\sigma} \\
    & = \sum_{t=1}^l \frac{\mu_{i,t}}{\sigma} - \frac{l \cdot \mu}{\sigma} 
\end{split}
\end{align}
}

\update{
Then let us fix $i$ and sum up all $v_i$ over all $i$, which yields:
\begin{align}\label{eq: centering_vec}
\begin{split}
    & \sum_{i=1}^{l} v_i = \sum_{i=1}^l \frac{\hat{v_i} - \mu}{\sigma} = \sum_{i=1}^l \frac{\hat{v_i}}{\sigma} - \frac{l \mu}{\sigma} 
\end{split}
\end{align}
}

\remove{
Then by integrating Equation \eqref{eq: base_concept_emb} and Equation \eqref{eq: all_mean} into the above formula, we can get:
\begin{align*}
    & \sum_{t=1}^l \mu_{i,t}' = \frac{1}{\sigma'}\left[\sum_{t=1}^l \frac{1}{l^{k-1}} \sum_{j_1}\sum_{j_2}\dots \sum_{j_{i-1}}\sum_{j_{i+1}}\dots \sum_{j_k} v_{j_1,j_2,j_3,\dots,j_{i-1},t,j_{i+1},\dots,j_k}\right.\\
    &\left. - \frac{1}{l^{k-1}} \sum_{j_1}\sum_{j_2}\dots \sum_{j_{i-1}}\sum_{j_{i+1}}\dots \sum_{j_k} v_{j_1,j_2,j_3,\dots,j_{i-1},t,j_{i+1},\dots,j_k}\right]\\
    & = \frac{1}{\sigma'}\left[\frac{1}{l^{k-1}} \sum_{j_1}\sum_{j_2}\dots \sum_{j_{i-1}}\sum_{t=1}^l\sum_{j_{i+1}}\dots \sum_{j_k} v_{j_1,j_2,j_3,\dots,j_{i-1},t,j_{i+1},\dots,j_k}\right.\\
    &\left. - \frac{1}{l^{k-1}} \sum_{j_1}\sum_{j_2}\dots \sum_{j_{i-1}}\sum_{j_{i+1}}\dots \sum_{j_k} v_{j_1,j_2,j_3,\dots,j_{i-1},t,j_{i+1},\dots,j_k}\right] = 0.
\end{align*}
}

\update{
Then by integrating Equation~\ref{eq: base_concept_emb} and Equation~\ref{eq: all_mean} into the above formula, we get:
\begin{align*}
    \sum_{i=1}^l v_i &= \sum_{i=1}^l \frac{\hat{v_i}}{\sigma} - \frac{l \mu}{\sigma}\\
    &= \frac{1}{\sigma} \sum_{i=1}^l \frac{1}{l'} \sum_{j=1}^{l'} v_{i,j} - \frac{l\mu}{\sigma}\\
    &= \frac{1}{\sigma l'}\sum_{i=1}^l\sum_{j=1}^{l'} v_{i,j} - \frac{1}{\sigma l'} \sum_{i=1}^l\sum_{j=1}^{l'} v_{i,j} \\
    &= 0\\
\end{align*}
We can equivalently show that $\sum_{j=1}^{l'} v'_j = 0$.
}

\update{
Therefore, the concept representations $v_i$ within the attribute $A$ are linearly dependent and the representations $v'_i$ within the attribute $A'$ are linearly dependent, meaning there exist concepts $c_i$ and $c_j$ such that $\langle v_i, v_j \rangle \neq 0$, and concepts $c'_k$ and $c'_m$ such that $\langle v_k', v_m' \rangle \neq 0$.
}

\remove{
This thus suggests that for $j_i, j_i \in \{1,2,\dots,l\}$, all $\mu_{j_i,i}'$ are correlated. If all pairs of $(\mu_{j_{i1},i}', \mu_{j_{i2},i}'), j_{i1},j_{i2} = 1,2,\dots, l$ are orthogonal, then we can obtain the following formula:
\begin{align*}
    \langle\mu_{i,t_1}', \sum_{t=1}^l \mu_{i,t}'\rangle = \langle\mu_{i,t_1}', \mu_{i,t_1}'\rangle = 0,
\end{align*}
thus indicating that all $\mu_{i,t}$ is 0. This is thus contradictory to our assumption that each $\mu_{i,t}$ is a non-zero vector. Therefore, there should exist at least one pair of $(\mu_{i,t_1}', \mu_{i,t_2}')$ which are not orthogonal. 
}

\update{
\textbf{Part 2: For all $c_1\in A$ and $c_2\in A'$, the representations of $c_1$ and $c_2$ are orthogonal with high probability.}
}

\update{
To prove that all concept representations from $A$ are orthogonal to all concept representations from $A'$ , we will show that the dot product between these two representations is zero.
Let $c_i\in A$ and $c'_j\in A'$ and $v_i, v'_j$ are the concept representations for $c_i$ and $c'_j$ respectively. We can expand the dot product as follows:
\begin{align*}
    \langle v_i, v'_j \rangle = \left\langle\frac{\hat{v_i}}{\sigma} - \frac{\mu}{\sigma},\frac{\hat{v}'_j}{\sigma} - \frac{\mu}{\sigma} \right\rangle
\end{align*}}

\update{
Then by integrating Equation~\ref{eq: base_concept_emb} and Equation~\ref{eq: all_mean} into the above formula, we can expand the above into the following:
\begin{align*}
    \langle v_i, v'_j \rangle = \frac{1}{\sigma^2} \left\langle \frac{1}{l'} \sum_{j=1}^{l'}v_{i,j} - \mu, \frac{1}{l}\sum_{i=1}^l v_{i,j} - \mu \right\rangle
\end{align*}
}

\remove{We next prove that arbitrary pairs of concept representations from two different attributes are orthogonal with high probability. To demonstrate this, we calculate the dot product between $\mu_{i_1,t_1}'$ and $\mu_{i_2,t_2}'$ which represents two concepts from attribute $i_1$ and $i_2$ respectively:
\begin{align*}
    & \langle\mu_{i_1,t_1}', \mu_{i_2,t_2}'\rangle = \langle\frac{\mu_{i,t}}{\sigma'} - \frac{\mu}{\sigma'},\frac{\mu_{i,t}}{\sigma'} - \frac{\mu}{\sigma'} \rangle \\
    & =\frac{1}{\sigma'^2} \frac{1}{l^{k}} \langle l \sum_{j_1}\sum_{j_2}\dots \sum_{j_{i_1-1}}\sum_{j_{i_1+1}}\dots \sum_{j_k} v_{j_1,j_2,j_3,\dots,j_{i_1-1},t_1,j_{i_1+1},\dots,j_k} - \sum_{j_1}\sum_{j_2}\dots \sum_{j_k} v_{j_1,j_2,j_3,\dots,j_k}, \\
    & l \sum_{j_1}\sum_{j_2}\dots \sum_{j_{i_2-1}}\sum_{j_{i_2+1}}\dots \sum_{j_k} v_{j_1,j_2,j_3,\dots,j_{i_2-1},t_2,j_{i_2+1},\dots,j_k} - \sum_{j_1}\sum_{j_2}\dots \sum_{j_k} v_{j_1,j_2,j_3,\dots,j_k}\rangle 
\end{align*}
}

\update{We note that for arbitrary pairs of $v_{i,j}$ and $v_{i',j'}$ with $i \neq i'$ or $j\neq j'$, since they are two different random vectors sampled from a spherical normal distribution $N(\mathbf{0}, \mathbf{I}^d)$, their dot product is $o(d)$ according to Equation~\ref{eq: inner_prod_bound}. Therefore, through some linear algebraic operations, the above formula could be reformulated as follows:
\begin{align*}
    \langle v_i, v'_j \rangle &= \frac{1}{\sigma^2} \left\langle \frac{1}{l'} \sum_{s=1}^{l'}v_{i,s} - \mu, \frac{1}{l}\sum_{t=1}^l v_{t,j} - \mu \right\rangle \\
    &= \frac{1}{\sigma^2} \left\langle \frac{1}{l'} \sum_{s=1}^{l'}v_{i,s} - \frac{1}{ll'}\sum_{t,s}v_{t,s}, \frac{1}{l}\sum_{t=1}^l v_{t,j} - \frac{1}{ll'}\sum_{t,s} v_{t,s} \right\rangle \\
    &= \frac{1}{\sigma^2 ll'} \left\langle \sum_{s=1}^{l'}v_{i,s} - \frac{1}{l}\sum_{t,s}v_{t,s}, \sum_{t=1}^l v_{t,j} - \frac{1}{l'}\sum_{t,s} v_{t,s} \right\rangle \\
    &= \frac{1}{\sigma^2 ll'} \left[ \sum_{s=1}^{l'}v_{i,s} \sum_{t=1}^l v_{t,j} - \frac{1}{l'}\sum_{s=1}^{l'}v_{i,s}\sum_{t,s}v_{t,s} - \frac{1}{l}\sum_{t=1}^{l}v_{t,j}\sum_{t,s}v_{t,s} + \frac{1}{ll'}\sum_{t,s}v_{t,s} \sum_{t,s}v_{t,s}\right] \\
    &= \frac{1}{\sigma^2 ll'} \left[ \|v_{i,j}\|^2 - \frac{1}{l'}\sum_{s=1}^{l'}\|v_{i,s}\|^2 - \frac{1}{l}\sum_{t=1}^l\|v_{t,j}\|^2 + \frac{1}{ll'}\sum_{t,s}\|v_{t,s}\|^2\right] + o(d)
\end{align*}
in which $o(d)$ is derived by applying Equation~\ref{eq: inner_prod_bound} to all the cross terms of the form $\langle v_{i,j}, v_{i',j'} \rangle$ where at least one pair of $i, i'$ and $j, j'$ are different.
}

\update{We can further simplify this expression using Lemma \ref{lemma:gaussian_bound} which says that for each vector $x$ randomly sampled from $N(\mathbf{0}, \mathbf{I}^d)$, its norm is bounded by $[\sqrt{d} - \epsilon, \sqrt{d} + \epsilon]$ with high probability, which applies to each $v_{i, j}$. Therefore, we can bound the above equation by:
\begin{align*}
    \langle v_i, v'_j \rangle &\leq \frac{1}{\sigma^2 ll'} \left[ (\sqrt{d}+\epsilon)^2 - \frac{1}{l'}l'(\sqrt{d}-\epsilon)^2 - \frac{1}{l}l(\sqrt{d}-\epsilon)^2 + \frac{1}{ll'}ll'(\sqrt{d} +\epsilon)^2 \right] o(d)\\
    &= \frac{8\sqrt{d}\epsilon}{\sigma^2 ll'} + o(d)
\end{align*}
}

Similarly, we can prove that \edit{$\langle v_i, v'_j \rangle \geq -\frac{8\sqrt{d}\epsilon}{\sigma^2 ll'} + o(d)$}, so we can conclude that

\remove{\begin{align}\label{eq: dot_prod_bound}
\langle\mu_{i_1,t_1}', \mu_{i_2,t_2}'\rangle = o(d)   
\end{align}}
\update{\begin{align}\label{eq: dot_prod_bound}
| \langle v_i, v'_j\rangle | = o(d)   
\end{align}}

Our goal is to get a bound on the cosine similarity of $v_i$ and $v'_j$ to show that it is zero. The cosine similarity is written \edit{$S_\text{cos} (v_i, v'_j) = \frac{\langle v_i, v'_j\rangle}{\|v_i\| \|v'_j\|}$, so we have a bound on the numerator, but we now want a bound on the terms in the denominator}. We can compute the norm of \remove{$\mu_{i_1,t_1}'$}\update{$v_i$ and $v'_j$} and follow the same derivation as above by leveraging Equation~\ref{eq: inner_prod_bound}, which results in:
\remove{
\begin{align*}
    \|v_i\|_2^2 & = \langle v_i, v_i\rangle \\
    &= \frac{1}{\sigma^2 ll'} \left\langle \sum_{j=1}^{l'}v_{i,j} - \frac{1}{l'}\sum_{i,j}v_{i,j}, \sum_{j=1}^{l'} v_{i,j} - \frac{1}{l'}\sum_{i,j} v_{i,j} \right\rangle \\
    &= \frac{1}{\sigma^2 ll'} \left[ \sum_{j=1}^{l'}v_{i,j} \sum_{i=1}^l v_{i,j} - \frac{1}{l'}\sum_{j=1}^{l'}v_{i,j}\sum_{i,j}v_{i,j} - \frac{1}{l}\sum_{i=1}^{l}v_{i,j}\sum_{i,j}v_{i,j} + \frac{1}{ll'}\sum_{i,j}v_{i,j} \sum_{i,j}v_{i,j}\right] \\
    &= \frac{1}{\sigma^2 ll'} \left[ \|v_{i,j}\|^2 - \frac{1}{l'}\sum_{j=1}^{l'}\|v_{i,j}\|^2 - \frac{1}{l}\sum_{i=1}^l\|v_{i,j}\|^2 + \frac{1}{ll'}\sum_{i,j}\|v_{i,j}\|^2\right] + o(d)
\end{align*}
This formula could then be lower bounded by:
\begin{align}\label{eq: norm_prod}
    \|\mu_{i_1,t_1}'\|_2^2 & \geq 2 l^{k} (d - 2\sqrt{d} \epsilon + \epsilon^2) + o(d) = 2l^k d + o(d)
\end{align}}
\update{
\begin{align*}
    \|v_i\|_2^2 & = \langle v_i, v_i\rangle \\
    &= \frac{1}{\sigma^2 l'^2} \left\langle \sum_{s=1}^{l'}v_{i,s} - \frac{1}{l}\sum_{t,s}v_{t,s}, \sum_{s=1}^{l'} v_{i,s} - \frac{1}{l}\sum_{t,s} v_{t,s} \right\rangle \\
    &= \frac{1}{\sigma^2 l'^2} \left[ \sum_{s=1}^{l'}v_{i,s} \sum_{s=1}^{l'} v_{i,s} - 2\frac{1}{l}\sum_{s=1}^{l'}v_{i,s}\sum_{t,s}v_{t,s} + \frac{1}{l^2}\sum_{t,s}v_{t,s} \sum_{t,s}v_{t,s}\right] \\
    &= \frac{1}{\sigma^2 l'^2} \left[ \sum_{s=1}^{l'} \|v_{i,s}\|^2 - \frac{2}{l}\sum_{s=1}^{l'} \|v_{i,s}\|^2 + \frac{1}{l^2}\sum_{t,s}\|v_{t,s}\|^2\right] + o(d)
\end{align*}
Similarly, we can get the following:
\begin{align*}
    \|v'_j\|_2^2 & = \frac{1}{\sigma^2 l^2} \left[ \sum_{t=1}^{l} \|v_{t,j}\|^2 - \frac{2}{l'}\sum_{t=1}^{l'} \|v_{t,j}\|^2 + \frac{1}{l'^2}\sum_{t,s}\|v_{t,s}\|^2\right] + o(d)
\end{align*}
By Lemma~\ref{lemma:gaussian_bound}, the norm of each \edit{$v_{i,j}$} is bounded by $\sqrt{d} - \epsilon$ and $\sqrt{d} + \epsilon$ with high probability, so the above formula can be bounded by:
\begin{align*}
    \frac{1}{\sigma^2 ll'}((l-1)d - (2l+6)\sqrt{d}\epsilon +(l-1)\epsilon^2) + o(d) \leq \|v_i\|_2^2 & \leq \frac{1}{\sigma^2 ll'}((l-1)d + (2l +6)\sqrt{d}\epsilon + (l-1)\epsilon^2) + o(d),
\end{align*}}

\update{
Therefore, 
\begin{align}\label{eq: norm_prod}
    \|v_i\|_2^2 = O(d)
\end{align}
and we can equivalently show that $\|v'_j\| = O(d)$.
}

\remove{
This leverages the fact that each $\|v_{j_1,j_2,j_3,\dots,j_k}\|$ is bounded by $[\sqrt{d} - \epsilon, \sqrt{d} + \epsilon]$ with high probability. The above formula also holds for $\|\mu_{i_2,t_2}'\|_2^2$. As a consequence, the cosine similarity between $\mu_{i_1,t_1}'$ and $\mu_{i_2,t_2}'$ is bounded by:
\begin{align*}
    \text{cosine}(\mu_{i_1,t_1}', \mu_{i_2,t_2}') = \frac{\langle \mu_{i_1,t_1}', \mu_{i_2,t_2}' \rangle}{\|\mu_{i_1,t_1}'\|\cdot \|\mu_{i_2,t_2}'\|} \leq \frac{o(d)}{2l^k d + o(d)},
\end{align*}
which thus approaches zero as $d$ increases.}
\update{
As a consequence, we can now calculate the cosine similarity between $v_i$ and $v'_j$:
\begin{align*}
    S_\text{cos}(v_i, v'_j) = \frac{\langle v_i, v'_j \rangle}{\|v_i\|\cdot \|v'_j\|} = \frac{o(d)}{O(d)} = o(1),
\end{align*}
which means that this converges to zero as desired.}
\end{proof}

\update{
\begin{corollary}
\label{thm:reverse-corollary}
Given Theorem~\ref{theorem}, for the representation of the composite concepts $v_{i,j}$, it can be (approximately) decomposed into the linear combinations of the representations of the base concepts (after the centering operation), $v_i, v_j$ but is orthogonal to the representations of other base concepts with high probability. In other words, compositionality holds with high probability.
\end{corollary}
}

\begin{proof}
\update{To prove this, let us consider the cosine similarity between $v_{i, j}$ and $v_t$.}

\update{According to Equation~\ref{eq: base_concept_emb}, we first compute the inner product between these two vectors, i.e.:
    \begin{align}\label{eq: decompose_correlation}
        \langle v_{i,j}, v_t\rangle = \frac{1}{l'} \sum_{n=1}^{l'} \langle v_{i,j} , v_{t,n} \rangle,
    \end{align}
}

\update{Depending on whether $t = i$ or not, there are two different cases. }

\paragraph{Case 1: $t \neq i$}

\update{Note that according to Lemma~\ref{lemma:gaussian_bound}, since $v_{i,j}$ and  $v_{t,n}$ are twowvectors randomly sampled from the spherical normal distribution, their inner product is $o(d)$. Therefore, the above inner product between $v_{i,j}$ and $v_{t}$ becomes:
\begin{align*}
    \langle v_{i, j}, \mu_{t}\rangle = o(d).
\end{align*}
}

\update{Also note that according to Equation~\ref{eq: centered_transform}, $v_t = \frac{\hat{v_t} - \mu}{\sigma}$, we thus need to leverage this equation to derive the inner product between $v_{i,j}$ and $v_t$. Furthermore, according to \eqref{eq: all_mean}, $\mu$ is the mean of all the representations of the composite concepts, which are all randomly sampled from a spherical normal distribution. Therefore, $\mu$ is approaching 0 with high probability and thus the following equation holds with high probability:
\begin{align*}
    \langle v_{i,j}, v_t\rangle = \langle v_{i,j}, \frac{\hat{v_t} - \mu}{\sigma} \rangle = \langle v_{i,j}, \frac{\hat{v_t}}{\sigma} \rangle = o(d), t \neq i.
\end{align*}
}

\update{In addition, according to Lemma~\ref{lemma:gaussian_bound} and Equation~\ref{eq: norm_prod}, the norms of $v_{i,j}$ and $v_t$ are both $O(\sqrt{d})$. Therefore, the cosine similarity between $v_{i,j}$ and $v_t$ :
\begin{align*}
    \text{cosine}(v_{i,j}, v_t) = \frac{\langle v_{i,j}, v_t\rangle}{\|v_{i,j}\|\cdot\|v_t\|} = \frac{o(d)}{\|v_{i,j}\|\cdot\|v_t\|} = \frac{o(d)}{O(d)} = o(1).
\end{align*}
}

\update{Intuitively speaking, this indicates that for the representation of a composite concept $v_{i,j}$, it is not correlated with the representation of a base concept that does not appear in this composite concept with high probability. For example, this could mean that the representation of the composite concept $\{c_{\text{red}},c_{\text{sphere}}\}$ is not correlated to the representation of the concept $c_{\text{blue}}$, which is intuitively true. }

\paragraph{Case 2: $t=i$} 
\update{In Equation~\ref{eq: decompose_correlation}, according to Lemma~\ref{lemma:gaussian_bound}, the inner product between $v_{i,j}$ and most $v_{t,m}$ is $o(d)$ except when $j=m$. Therefore, Equation~\ref{eq: decompose_correlation} becomes:
\begin{align*}
    \langle v_{i,j}, v_t\rangle = \|v_{i,j}\|_2^2 + o(d),
\end{align*}
}

\update{
Then according to Lemma~\ref{lemma:gaussian_bound}, since $\|v_{i,j}\|$ is approaching $\sqrt{d}$, then the above formula is transformed to:
\begin{align*}
    \langle v_{i,j}, v_t\rangle = O(d),
\end{align*}
}

\update{
Then according to Lemma~\ref{lemma:gaussian_bound} and Equation~\ref{eq: norm_prod}, the norms of $v_{i,j}$ and $v_t$ are both $O(\sqrt{d})$. Therefore, the cosine similarity between $v_{i,j}$ and $v_t$ is:
\begin{align*}
    \text{cosine}(v_{i,j}, v_t) = \frac{\langle v_{i,j}, v_t\rangle}{\|v_{i,j}\|\cdot\|v_t\|} = \frac{O(d)}{\|v_{i,j}\|\cdot\|v_t\|} = \frac{O(d)}{O(d)} = O(1),
\end{align*}
which is thus a nonzero value.
}

\update{As indicated by the above analysis, we can conclude that each $v_{i,j}$ is only correlated to the representation of the base concepts $v_i$, and $v'_j$. Since the representations of those base concepts are from different attributes, thus orthogonal to each other, then we can regard them as the basis vectors in the vector space, which can then be linearly combined to approximately reconstruct $v_{i,j}$, i.e.:
\begin{align*}
    v_{i,j} = \text{cosine}(v_{i,j}, v_i)v_i + \text{cosine}(v_{i,j}, v'_j)v'_j
\end{align*}
This thus matches the definition of the compositionality (see Definition \ref{def:composition}). }

\end{proof}

\edit{
\begin{theorem}
    For some dataset, consider two attributes $A$ and $A'$ where we have $l$ concepts for $A$, $c_1,\dots,c_l$, and $l'$ concepts for $A'$, $c'_1,\dots,c'_{l'}$. Define normalized concept representations $v_1,\dots,v_l$ and $v'_1,\dots,v'_{l'}$ for the concepts in $A$ and $A'$ such that $v_i$ is orthogonal to $v'_j$ for all $i$ and $j$ and for $v_i$ and samples $x$ and $x'$ such that $x$ has concept $c_i$ and $x'$ does not, then $S_\text{cos}(x, v_i) > S_\text{cos}(x', v_i)$. Then the concept representations are compositional.
\end{theorem}
\begin{proof}
    Let $v_i$ be the concept representation for $c_i$ and $v'_j$ be the concept representation for $c'_j$. We are given that for any two samples $x$ and $x'$ with and without concept $c_i$ respectively, $S_\text{cos}(x, v_i) > S_\text{cos}(x', v_i)$ and similarly for any two samples $x$ and $x'$ with and without concept $c'_j$ respectively, $S_\text{cos}(x, v'_j) > S_\text{cos}(x', v'_j)$. We will show that a concept representation for $c_{i, j}$, the composition of concept $c_i$ and $c'_j$, exists and is represented by $v_{i,j}=v_i+v'_j$.

    Let $v_{i,j} = v_i + v'_j$. We will show that this concept can perfectly rank samples with the concept $c_{i,j}$. Since $v_i$ and $v'_j$ result in perfect rankings, for all $x, x'$ such that $x$ has $c_i$ and $x'$ does not, $S_\text{cos}(x, v_i) - S_\text{cos}(x', v_i) > 0$. Similarly, for any $x, x'$ such that $x$ has $c'_j$ and $x'$ does not, $S_\text{cos}(x, v'_j) - S_\text{cos}(x', v'_j) > 0$.

    Now let, $x, x'$ be such that $x$ has concept $c_{i,j}$ and $x'$ does not. We can write the following:
    \begin{align*}
         S_\text{cos}(x, v_i + v'_j) &= \frac{\langle x, v_i + v'_j \rangle}{\|x\|\|v_i +v'_j\|}\\
         &= \frac{\langle x, v_i\rangle + \langle x, v'_j \rangle}{\|x\|\sqrt{2}} & \text{Since $\langle v_i,v'_j\rangle = 0$, $\langle v_i,v_i\rangle = 1$, and $\langle v'_j,v'_j\rangle = 1$}\\
         &= \frac{1}{\sqrt{2}} (S_\text{cos}(x, v_i) + S_\text{cos}(x, v'_j))
    \end{align*}
    Therefore, we can now show that the concept score for the composed concept is larger for $x$ than $x'$:
    \begin{align*}
        S_\text{cos}(x, v_i + v'_j) - S_\text{cos}(x', v_i + v'_j) &= \frac{1}{\sqrt{2}} (S_\text{cos}(x, v_i) + S_\text{cos}(x, v'_j)) - \frac{1}{\sqrt{2}} (S_\text{cos}(x', v_i) + S_\text{cos}(x', v'_j))\\
        &= \frac{1}{\sqrt{2}} \left( (S_\text{cos}(x, v_i) - S_\text{cos}(x', v_i)) + (S_\text{cos}(x, v'_j) - S_\text{cos}(x', v'_j))\right)\\
        &> 0.
    \end{align*}
\end{proof}
}

\section{Compositionality of Ground-Truth Concepts}
The cosine similarities between concepts is shown for the CUB-sub and Truth-sub datasets in Figure~\ref{fig:gt-cosine}. We see similar findings as in Figure~\ref{fig:clevr-gt}.
\label{sec:comp-gt}
\begin{figure}
    \centering
    \begin{subfigure}[b]{0.4\textwidth}
        \includegraphics[width=\columnwidth]{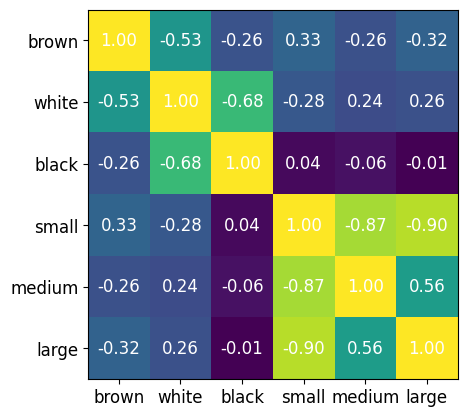}
        \caption{\cubsub{}}
        \label{fig:cubsub-gt}
    \end{subfigure}
    \hfill
    \begin{subfigure}[b]{0.4\textwidth}
        \includegraphics[width=\columnwidth]{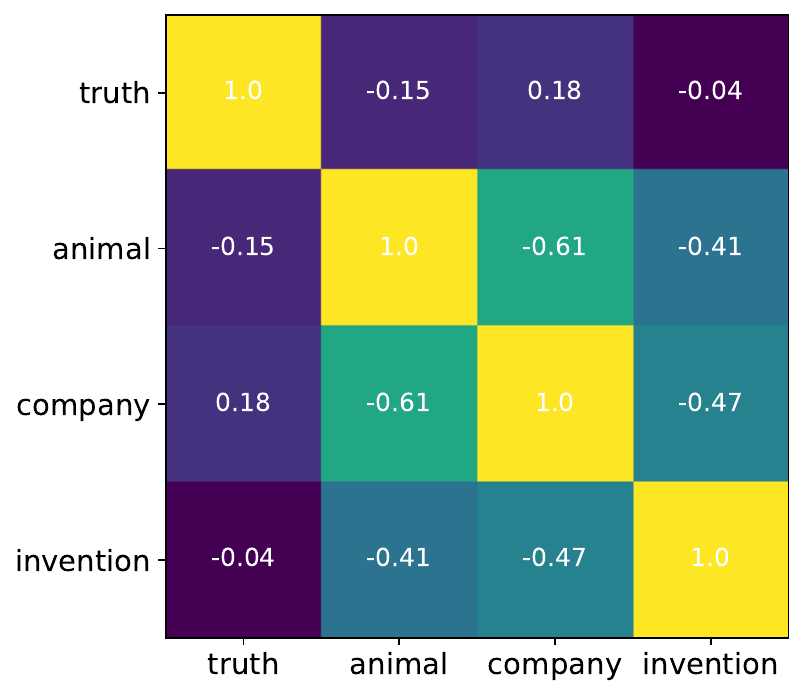}
        \caption{\nlpsub{}}
        \label{fig:truth-gt}
    \end{subfigure}
    \hfill
    \caption{Compositionality of Ground-Truth Concepts for the \cubsub{} and \nlpsub{} datasets.}
    \label{fig:gt-cosine}
\end{figure}

\section{Qualitative Examples}
We provide additional qualitative results for the CUB dataset in Figure~\ref{fig:cub_extra_qual}
and the ImageNet \citep{ILSVRC15} validation set in Figure~\ref{fig:imagenet_qual}. The concepts are named by manually looking at the top 20 images for each concept and coming up with a short description which is as specific as possible to the images while being general enough to apply to each image.

As an alternative to manual concept labelling, we also experimented with using a vision-text language model to automatically name concepts from their top 20 examples. We used GPT-4o \citep{gpt4o} to get concept labels. For each concept, we produce a single image containing the top 20 samples for the concept and we pass the image to GPT-4o with the following prompt:
\begin{verbatim}
You are given 20 images representing a single concept and your task is to label
the name of the concept from just the 20 images. First, output a detailed caption
for each image. Then output a concept name which is specific to the images but
summarizes what is common among all of them. For example, for images of red cars
in different environments and positions, the concept name could be 'Red cars'.
Output the name of the concept after 'Concept Name:'. 
\end{verbatim}

The labels for the additional CUB examples in Figure~\ref{fig:cub_extra_qual} are the following where each line labels a row of the figure:
\begin{verbatim}
Hummingbirds, Birds, Hummingbirds
Black Birds, Birds in Natural Habitats, Black Birds
Wrens, Birds with food in their beaks, Wrens
Seagulls, Birds with food in their beaks, Birds with fish in their beaks
\end{verbatim}

Similarly, the labels from GPT-4o for Figure~\ref{fig:imagenet_qual} are the following:
\begin{verbatim}
Dogs, Sleeping in various environments, Sleeping Dogs
Reptiles and Amphibians in Natural Habitats, Pairs of Dogs, Pairs of Animals
Wild Animals, Pairs of Dogs, Animals in Pairs
Waterfront Structures and Transportation,
    Outdoor Activities and Wildlife,
    British Heritage and Infrastructure
Tools and Objects in Close-Up,
    Laboratory and Scientific Equipment,
    Vintage and Everyday Objects
\end{verbatim}

\begin{figure}
    \centering
    \includegraphics[width=0.75\textwidth]{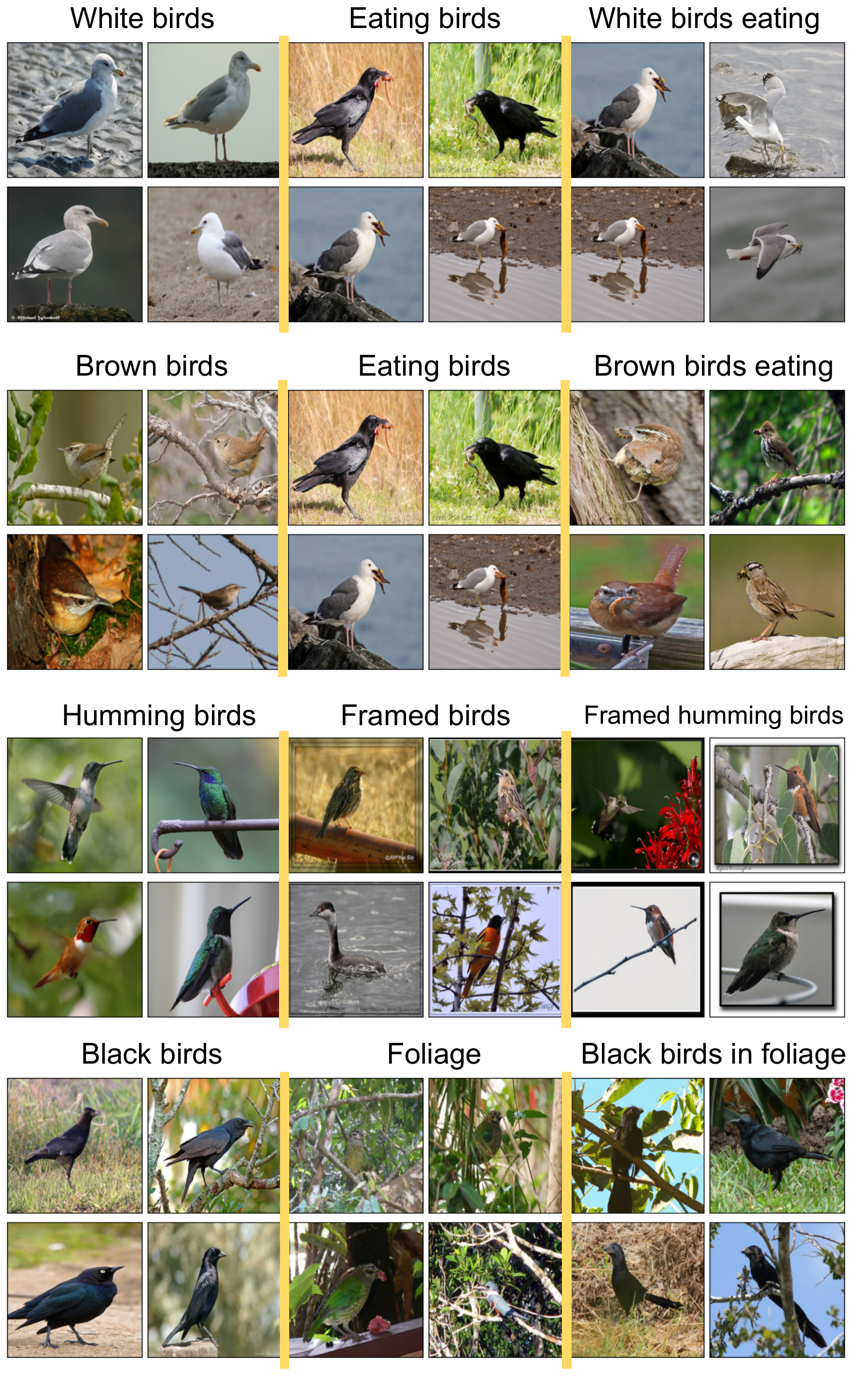}
    \caption{Additional CUB qualitative examples.}
    \label{fig:cub_extra_qual}
\end{figure}

\begin{figure}
    \centering
    \includegraphics[width=0.60\textwidth]{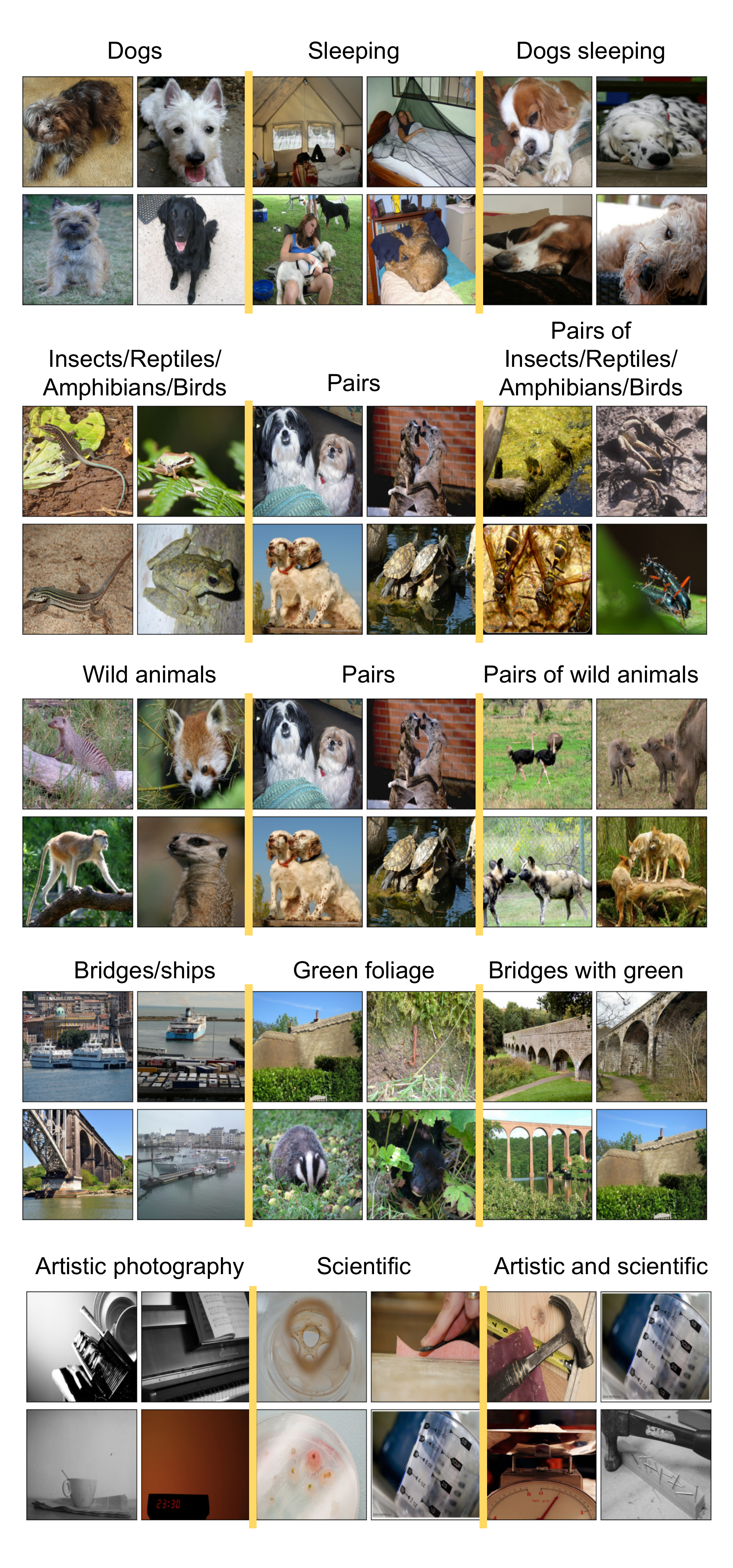}
    \caption{ImageNet qualitative examples.}
    \label{fig:imagenet_qual}
\end{figure}

\section{Additional quantitative results}

\subsection{Runtime analysis}
\begin{table}[h]
    \centering
    \caption{Runtimes in seconds}
    \label{tab:hyperparams}
    \begin{tabular}{lrrrrrr}
    \toprule
        Dataset & PCA & ACE & DictLearn & SemiNMF & CT & \ourmethod{} \\ \hline
        CLEVR & 0.10 $\pm$ 0.12 & 0.02 $\pm$ 0.00 & 28.65 $\pm$ 0.29 & 8.13 $\pm$ 1.03 & 63.66 $\pm$ 0.73 & 190.98 $\pm$ 2.38 \\
        CUB-sub & 0.11 $\pm$ 0.15 & 0.03 $\pm$ 0.01 & 14.38 $\pm$ 0.15 & 3.99 $\pm$ 0.09 & 6.89 $\pm$ 0.15 & 112.73 $\pm$ 2.67 \\
        CUB & 0.84$\pm$0.06 & 0.46$\pm$0.03 & 51.53$\pm$1.51 & 25.85$\pm$0.22 & 495.49$\pm$10.81 & 207.17$\pm$0.70 \\
        Truth-sub & 0.16 $\pm$ 0.03 & 0.06 $\pm$ 0.02 & 43.36 $\pm$ 4.35 & 29.83 $\pm$ 0.62 & 165.06 $\pm$ 1.21 & 316.45 $\pm$ 2.63 \\
        Truth & 1.10 $\pm$ 0.16 & 2.64 $\pm$ 0.09 & 88.81 $\pm$ 6.54 & 194.67 $\pm$ 10.18 & 712.16 $\pm$ 7.70 & 1574.88 $\pm$ 17.68 \\
        HAM & 1.89$\pm$0.03 & 2.97$\pm$0.03 & 367.67$\pm$8.71 & 165.80$\pm$2.22 & 693.73$\pm$1.88 & 7460.52$\pm$47.95 \\
        News & 3.28$\pm$0.72 & 25.75$\pm$2.39 & 241.75$\pm$38.70 & 934.69$\pm$117.66 & 431.78$\pm$7.11 & 7947.31$\pm$70.64 \\
    \bottomrule
    \end{tabular}
\end{table}

\subsection{Downstream performance error bars}
We include error bars for the downstream performance results using the greatest number of concepts in Table~\ref{tab:error_bars}.

\remove{
\begin{table}[h]
    \centering
    \caption{Error bars of the downstream performance (\%)}
    \label{tab:error_bars}
    \begin{tabular}{lrrrrrrr}
    \toprule
         Method & CUB & Truth & HAM & News\\
         \midrule
pca & 72.71$\pm$0.01 & \textbf{86.70$\pm$0.00} & 77.42$\pm$0.01 & \\
ace & 74.99$\pm$0.06 & 85.83$\pm$0.00 & 78.67$\pm$0.12 &\\
dictlearn & 75.33$\pm$0.07 & 85.97$\pm$0.00 & 79.65$\pm$0.01 &\\
seminmf & 75.81$\pm$0.11 & 86.55$\pm$0.00 &76.30$\pm$0.03 &\\
ct & 65.60$\pm$0.12 & 74.12$\pm$0.00 & 72.71$\pm$0.06 &\\
\ourmethod{} & \textbf{76.49$\pm$0.47} & 84.93$\pm$0.01 & \textbf{80.05$\pm$0.01}\\
    \bottomrule
    \end{tabular}
\end{table}
}

\begin{table}[h]
    \centering
    \caption{Error bars of the downstream performance (\%). Three decimal places are given when  necessary to show non-zero standard deviation.}
    \label{tab:error_bars}
    \begin{tabular}{lrrrrrrr}
    \toprule
         Method & CUB & Truth & HAM & News\\
         \midrule
PCA & 72.71$\pm$0.01 & 87.137$\pm$0.000 & 77.42$\pm$0.01 & 62.029$\pm$0.001\\
ACE & 74.99$\pm$0.06 & 87.161$\pm$0.001 & 78.67$\pm$0.12 & 57.019$\pm$0.004\\
DictLearn & 75.33$\pm$0.07 & 87.500$\pm$0.002 & 79.65$\pm$0.01 & 61.015$\pm$0.002\\
SemiNMF & 75.81$\pm$0.11 & 87.355$\pm$0.001&76.30$\pm$0.03 & \textbf{62.215$\pm$0.002}\\
CT & 65.60$\pm$0.12 & 84.520$\pm$0.004& 72.71$\pm$0.06 & 47.207$\pm$0.007\\
\ourmethod{} & \textbf{76.49$\pm$0.47} & \textbf{87.888$\pm$0.001} & \textbf{80.05$\pm$0.01} & 61.670$\pm$0.003\\
    \bottomrule
    \end{tabular}
\end{table}

\subsection{Ablation on regularization in \ourmethod{}}
To see the impact of the regularization step in the LearnSubspace step of \ourmethod{}, we performan an additional ablation on the CLEVR dataset. We compare \ourmethod{} without this regularization step to the full implementation of \ourmethod in Table~\ref{tab:regularization_ablate}, and we see that regularization improves all three metrics.

\begin{table*}
    \centering
    \caption{Regularization ablation on CLEVR.}
    \label{tab:regularization_ablate}
    \begin{tabular}{lrrr}
    \toprule
    Method & MAP & Comp. Score & Mean Cosine\\
    \midrule
    \ourmethod{} & \textbf{1.00 $\pm$ 0.00} & \textbf{3.41 $\pm$ 0.18} & \textbf{0.99 $\pm$ 0.00}\\
    \ourmethod{}-NoReg & 0.97 $\pm$ 0.03 & 3.81 $\pm$ 0.21 & 0.78 $\pm$ 0.09
\\
    \bottomrule
    \end{tabular}
\end{table*}

\subsection{Ablation on clustering loss function}
We perform an ablation on the use of the Silhouette score as our clustering loss. Instead of Silhouette we experiment with the cross entropy loss based on the technique from \citet{caron2018deep}, but our results in Table~\ref{tab:clevr_loss_ablate} show that the Silhouette results in better compositionality.

\begin{table*}
    \centering
    \caption{Loss function ablation on CLEVR.}
    \label{tab:clevr_loss_ablate}
    \begin{tabular}{llrrr}
    \toprule
    Dataset & Loss & MAP & Comp. Score & Mean Cosine\\
    \midrule
    CLEVR & Silhouette & \textbf{1.00 $\pm$ 0.00} & \textbf{3.41 $\pm$ 0.18} & \textbf{0.99 $\pm$ 0.00} \\
CLEVR & Cross Entropy & 0.94 $\pm$ 0.08 & 3.44 $\pm$ 0.14 & 0.89 $\pm$ 0.10 \\
\midrule
Truth-sub & Silhouette & \textbf{0.56 $\pm$ 0.02} & \textbf{3.68 $\pm$ 0.01} & \textbf{0.81 $\pm$ 0.01} \\
Truth-sub & Cross Entropy & 0.50 $\pm$ 0.04 & 3.94 $\pm$ 0.04 & 0.75 $\pm$ 0.02 \\
\midrule
CUB-sub & Silhouette & \textbf{0.65 $\pm$ 0.01} & \textbf{0.48 $\pm$ 0.00} & \textbf{0.77 $\pm$ 0.01} \\
CUB-sub & Cross Entropy & 0.62 $\pm$ 0.04 & 0.49 $\pm$ 0.00 & 0.76 $\pm$ 0.01 \\
    \bottomrule
    \end{tabular}
\end{table*}

\subsection{Ablation on attribute imbalance}
\begin{figure}
    \centering
    \includegraphics[width=0.5\linewidth]{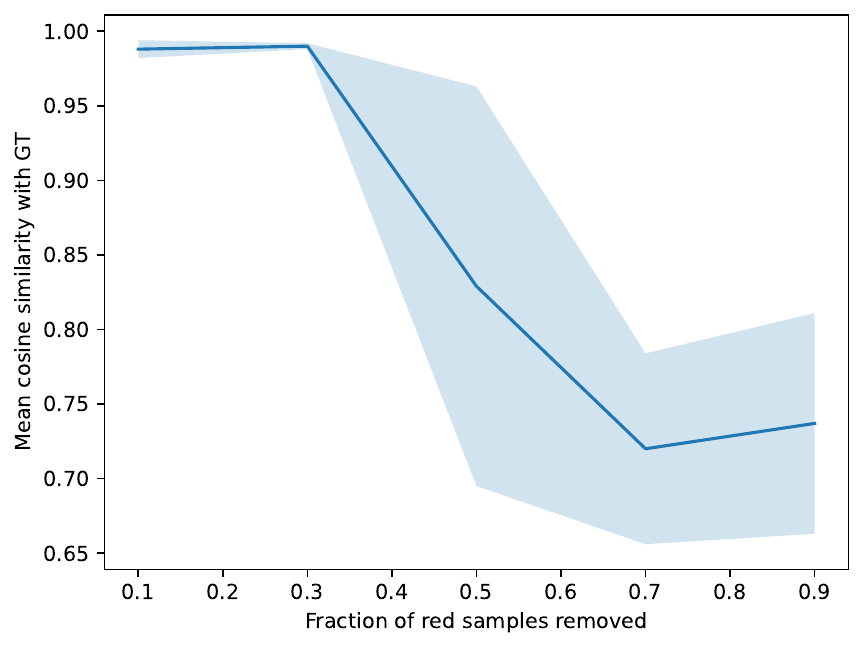}
    \caption{Cosine similarity between discovered ``red'' concept and the ground-truth ``red'' concept after removing a certain fraction of the red samples in the training set. As the attribute imbalance becomes larger, meaning there are less red samples than other colored samples, \ourmethod{} performs worse at finding the true red concept.}
    \label{fig:imbalance-ablate}
\end{figure}
We perform an ablation experiment on the effect of attribute imbalance by testing \ourmethod{}'s ability to recover the ground truth concepts on the CLEVR dataset after removing different fractions of samples labeled with the ``red'' concept. The results are shown in Figure~\ref{fig:imbalance-ablate} where we see that removing more red samples, which creates a greater imbalance, decreases the average cosine similarity of the discovered concepts with the ground truth.

\subsection{ROC-AUC Scores between Concept Representations and Ground-Truth}
\label{app:roc-auc}
The maximum ROC-AUC between the concept score and the true label for the ground-truth concepts is presented in Table~\ref{tab:clevr_auc} for CLEVR, Table~\ref{tab:cub_auc} for \cubsub{}, and Table~\ref{tab:nlp_auc} for \nlpsub{}.

\begin{table*}[h]
    \centering
    \caption{Max AUC score CLEVR v/s GT}
    \label{tab:clevr_auc}
    \begin{tabular}{l c c c c c c}
        \toprule
    
        Concepts & \ourmethod{} & ACE & ACE & PCA & DictLearn & SemiNMF \\
        \midrule
red & 1.000 & 0.765 & 0.728 & 0.985 & 0.757 & 0.793 \\
green & 1.000 & 0.771 & 0.711 & 0.996 & 0.797 & 0.818 \\
blue & 1.000 & 0.753 & 0.745 & 0.972 & 0.782 & 0.836 \\
sphere & 1.000 & 1.000 & 0.736 & 1.000 & 1.000 & 1.000 \\
cube & 1.000 & 0.998 & 0.742 & 0.971 & 0.994 & 0.999 \\
cylinder & 1.000 & 0.998 & 0.831 & 0.977 & 0.992 & 0.998 \\
(red and sphere) object & 0.987 & 0.993 & 0.911 & 0.950 & 0.978 & 0.983 \\
(red and cube) object & 0.923 & 0.999 & 1.000 & 0.965 & 0.983 & 0.999 \\
(red and cylinder) object & 0.899 & 0.940 & 0.932 & 0.964 & 0.998 & 0.943 \\
(green and sphere) object & 0.858 & 0.991 & 0.870 & 0.863 & 0.980 & 0.986 \\
(green and cube) object & 0.878 & 1.000 & 1.000 & 0.877 & 0.951 & 1.000 \\
(green and cylinder) object & 0.936 & 0.916 & 0.960 & 0.969 & 1.000 & 0.994 \\
(blue and sphere) object & 0.952 & 0.996 & 1.000 & 0.834 & 0.940 & 0.997 \\
(blue and cube) object & 0.878 & 1.000 & 1.000 & 0.973 & 0.842 & 0.978 \\
(blue and cylinder) object & 0.923 & 0.992 & 1.000 & 0.990 & 0.995 & 0.995 \\
  
        \bottomrule
    \end{tabular}
\end{table*}

\begin{table}[h]
    \centering
    \caption{ROC AUC of baseline methods on recovering the labeled concepts.}
    \label{tab:cub_auc}
    \begin{tabular}{lrrrrrr}
    \toprule
        Method & Brown & White & Black & Small & Medium & Large\\
        \midrule
        GT & 0.984 & 0.999 & 0.998 & 1.000 & 0.923 & 0.847 \\
        \midrule
        PCA & 0.881 & 0.985 & 0.931 & 0.997 & 0.886 & 0.677 \\ 
        ACE & 0.895 & 0.785 & 0.677 & 0.726 & 0.584 & 0.678 \\
        DictLearn & 0.849 & 0.645 & 0.650 & 0.702 & 0.519 & 0.551 \\
        SemiNMF & 0.086 & 0.164 & 0.099 & 0.116 & 0.066 & 0.168 \\
        CT & 0.923 & 0.837 & 0.887 & 0.926 & 0.754 & 0.736 \\
        Random & 0.867 & 0.933 & 0.855 & 0.888 & 0.849 & 0.723 \\
        \midrule
        \ourmethod{} & 0.894 & 0.834 & 0.710 & 0.743 & 0.656 & 0.661 \\
    \bottomrule
    \end{tabular}
\end{table}

\begin{table}[h]
    \centering
    \caption{ROC AUC of baseline methods on recovering the labeled concepts.}
    \label{tab:nlp_auc}
    \begin{tabular}{lrrrr}
    \toprule
         Method &  Truth & Animal & Company & Invention\\
         \midrule
         GT & 0.91 & 1.00 & 1.00 & 1.00\\
         PCA & 0.829 & 0.917 & 0.832 & 0.863\\
         ACE & 0.777 & 0.999 & 0.941 & 0.795\\
         DictLearn & 0.353 & 0.734 & 0.627 & 0.539\\
         SemiNMF & 0.759 & 0.708 & 0.629 & 0.521\\
         \ourmethod{} & 0.91 & 1.00 & 0.96 & 0.78\\
    \bottomrule
    \end{tabular}
\end{table}

\begin{table*}[h]
    \centering
    \caption{Max AUC score CLEVR v/s GT ViT}
    \label{tab:clevr_auc_vit}
    \begin{tabular}{l c c c c c}
        \toprule
    
        Concepts & \ourmethod{} & ACE & PCA & DictLearn & SemiNMF \\
        \midrule
red & 1.000 & 0.735 & 0.945 & 0.710 & 0.712 \\
green & 1.000 & 0.711 & 0.922 & 0.716 & 0.680 \\
blue & 1.000 & 0.642 & 0.995 & 0.704 & 0.629 \\
sphere & 1.000 & 0.610 & 1.000 & 1.000 & 1.000 \\
cube & 1.000 & 0.735 & 0.970 & 0.999 & 1.000 \\
cylinder & 1.000 & 0.695 & 1.000 & 1.000 & 1.000 \\
(red and sphere) object & 0.972 & 1.000 & 0.980 & 0.997 & 0.991 \\
(red and cube) object & 0.884 & 0.720 & 0.881 & 0.992 & 0.967 \\
(red and cylinder) object & 0.933 & 0.837 & 0.962 & 0.998 & 1.000 \\
(green and sphere) object & 0.904 & 1.000 & 0.923 & 0.998 & 0.985 \\
(green and cube) object & 0.913 & 0.731 & 0.886 & 0.920 & 0.937 \\
(green and cylinder) object & 0.895 & 0.660 & 0.866 & 0.988 & 0.939 \\
(blue and sphere) object & 0.939 & 0.844 & 0.970 & 0.954 & 0.949 \\
(blue and cube) object & 0.825 & 0.770 & 0.905 & 0.838 & 0.851 \\
(blue and cylinder) object & 0.854 & 0.766 & 0.842 & 0.913 & 0.875 \\
        \bottomrule
    \end{tabular}

\end{table*}

\subsection{The analysis of the cosine similarity score between learned concept representations and ground-truth}

We further break down the results reported in Table \ref{tab:cosine_mean} average cosine similarity between the learned concept representation and the ground-truth concept representations.

\subsection{Ablation studies on other pretrained models}
\label{app:ablation}
Recall that in the experiment section, we primarily focus on discovering concepts from pretrained CLIP model. In this section, we study with different choices of pretrained models, can we obtain similar results as that in Section \ref{sec:exp}? 

To answer this question, we leverage vision transformer (ViT), another widely used pretrained vision model, to repeat the experiments on CLEVR dataset. The results are summarized in Table \ref{tab:vit}-\ref{tab:resnet}. The results from these tables maintain the same trends as the one shown in Section \ref{sec:exp}.

\begin{table}[h]
    \centering
    \caption{ViT results on CLEVR}
    \label{tab:vit}
    \begin{tabular}{lrrr}
    \toprule
         Method &  MAP & Comp. Score & Mean Cosine\\
         \midrule
GT & \cellcolor{pink!100} 1.00 $\pm$ 0.00 & \cellcolor{pink!100} 3.69 $\pm$ 0.00 & \cellcolor{pink!100} 1.00 $\pm$ 0.00 \\
\midrule
PCA & \cellcolor{pink!85} 0.90 $\pm$ 0.00 & \cellcolor{pink!42} 4.33 $\pm$ 0.00 & \cellcolor{pink!42} 0.64 $\pm$ 0.00 \\
ACE & \cellcolor{pink!42} 0.70 $\pm$ 0.05 & \cellcolor{pink!28} 4.36 $\pm$ 0.11 & \cellcolor{pink!71} 0.67 $\pm$ 0.00 \\
DictLearn & \cellcolor{pink!71} 0.80 $\pm$ 0.04 & \cellcolor{pink!71} 3.98 $\pm$ 0.06 & \cellcolor{pink!85} 0.70 $\pm$ 0.01 \\
SemiNMF & \cellcolor{pink!57} 0.76 $\pm$ 0.01 & \cellcolor{pink!57} 4.29 $\pm$ 0.02 & \cellcolor{pink!57} 0.67 $\pm$ 0.00 \\
CT & \cellcolor{pink!14} 0.58 $\pm$ 0.05 & \cellcolor{pink!0} 6.26 $\pm$ 0.00 & \cellcolor{pink!14} 0.04 $\pm$ 0.01 \\
Random & \cellcolor{pink!28} 0.64 $\pm$ 0.03 & \cellcolor{pink!14} 6.26 $\pm$ 0.00 & \cellcolor{pink!28} 0.05 $\pm$ 0.00 \\
\ourmethod{} & \cellcolor{pink!100} 1.00 $\pm$ 0.00 & \cellcolor{pink!85} 3.87 $\pm$ 0.25 & \cellcolor{pink!100} 1.00 $\pm$ 0.00 \\
    \bottomrule
    \end{tabular}
\end{table}

\begin{table}[h]
    \centering
    \caption{ResNet-50 results on CLEVR}
    \label{tab:resnet}
    \begin{tabular}{lrrr}
    \toprule
         Method &  MAP & Comp. Score & Mean Cosine\\
         \midrule
GT & \cellcolor{pink!100} 0.95 $\pm$ 0.00 & \cellcolor{pink!100} 1.77 $\pm$ 0.00 & \cellcolor{pink!100} 1.00 $\pm$ 0.00 \\
\midrule
PCA & \cellcolor{pink!85} 0.90 $\pm$ 0.00 & \cellcolor{pink!28} 2.08 $\pm$ 0.00 & \cellcolor{pink!42} 0.58 $\pm$ 0.00 \\
ACE & \cellcolor{pink!71} 0.77 $\pm$ 0.04 & \cellcolor{pink!71} 1.92 $\pm$ 0.02 & \cellcolor{pink!57} 0.68 $\pm$ 0.01 \\
DictLearn & \cellcolor{pink!57} 0.71 $\pm$ 0.08 & \cellcolor{pink!57} 1.95 $\pm$ 0.11 & \cellcolor{pink!71} 0.68 $\pm$ 0.01 \\
SemiNMF & \cellcolor{pink!42} 0.64 $\pm$ 0.00 & \cellcolor{pink!42} 2.01 $\pm$ 0.01 & \cellcolor{pink!85} 0.69 $\pm$ 0.00 \\
CT & \cellcolor{pink!28} 0.63 $\pm$ 0.08 & \cellcolor{pink!14} 2.83 $\pm$ 0.00 & \cellcolor{pink!28} 0.03 $\pm$ 0.00 \\
Random & \cellcolor{pink!14} 0.57 $\pm$ 0.03 & \cellcolor{pink!0} 2.83 $\pm$ 0.00 & \cellcolor{pink!14} 0.03 $\pm$ 0.00 \\
\ourmethod{} & \cellcolor{pink!100} 0.92 $\pm$ 0.01 & \cellcolor{pink!85} 1.78 $\pm$ 0.01 & \cellcolor{pink!100} 0.96 $\pm$ 0.04 \\
    \bottomrule
    \end{tabular}
\end{table}

\begin{table}[h]
    \centering
    \caption{Cosine similarity of baseline methods for recovering the labeled concepts.}
    \label{tab:cosine_sim}
    \begin{tabular}{lrrrr}
    \toprule
         Method &  Truth & Animal & Company & Invention\\
         \midrule
         PCA & 0.367 & 0.139 & 0.688 & 0.583\\
         ACE & 0.244 & 0.956 & 0.733 & 0.642\\
         DictLearn & 0.760 & 0.988 & 0.917 & 0.879\\
         SemiNMF & 0.824 & 0.898 & 0.931 & 0.725\\
         \ourmethod{} & 0.90 & 0.94 & 0.85 & 0.64\\
    \bottomrule
    \end{tabular}
\end{table}

\section{Dataset Details}\label{app:datasets}
We provide the details for all datasets in Table~\ref{tab:datasets}.

\begin{table}[h]
    \centering
    \caption{Dataset details for all experiments}
    \label{tab:datasets}
    \begin{tabular}{lrrl}
    \toprule
         Dataset &  Total Samples & Number of GT Concepts & Modality\\
         \midrule
         CLEVR & 1001 & 6 & Image\\
         CUB &11788& NA & Image\\
         CUB-sub &261 & 6 & Image\\
         Truth & 4127 & NA & Text\\
         Truth-sub & 1125 & 4 & Text\\
         HAM &10015& NA & Image\\
         News & 18846 & NA & Text\\
    \bottomrule
    \end{tabular}
\end{table}

\section{Hyperparameters}
The hyperparameters of all experiments are given in Table~\ref{tab:hyperparams}.
\begin{table}[h]
    \centering
    \caption{Hyperparameters}
    \label{tab:hyperparams}
    \begin{tabular}{lrrrr}
    \toprule
         Dataset & $K$ & $M$ & learning rate\\
         \midrule
         CLEVR & 3 & 3 & 0.001\\
         CUB &20&5 &0.001\\
         CUB-sub &5&4&0.1\\
         Truth & 12 & 10 & 0.001\\
         Truth-sub & [4, 2, 3] & 3 & 0.001\\
         HAM &20 & 25&0.02 \\
         News & 15 & 30 & 0.001\\
    \bottomrule
    \end{tabular}
\end{table}

\end{document}